\pdfoutput=1
\documentclass[11pt]{article}

\usepackage[preprint]{acl}

\usepackage{times}
\usepackage{latexsym}

\usepackage[T1]{fontenc}
\usepackage[utf8]{inputenc}

\usepackage{microtype}

\usepackage{inconsolata}

\usepackage{graphicx}

\usepackage{hyperref}
\usepackage{url}
\usepackage{booktabs}       % professional-quality tables
\usepackage{amsfonts}       % blackboard math symbols
\usepackage{amsthm}
\usepackage{amsmath}
\usepackage{amssymb}
\usepackage{nicefrac}       % compact symbols for 1/2, etc.
\usepackage{microtype}      % microtypography
\usepackage[ruled,linesnumbered]{algorithm2e}
\usepackage[noend]{algpseudocode}
\usepackage{graphicx}
\usepackage{wrapfig}
\usepackage{float, caption, subcaption}
\usepackage{mathtools}
\usepackage{enumerate}
\usepackage{caption}
\usepackage{cases}
\usepackage{mathrsfs}
\usepackage{verbatim}
\usepackage[most]{tcolorbox}
\usepackage{empheq}
\usepackage[capitalize]{cleveref}
\usepackage{pifont}
\usepackage{multirow}
\usepackage{colortbl}  
\usepackage{stfloats}
\usepackage{xcolor}

\usepackage{amsmath,amsfonts,bm}

\def\eqref#1{equation~\ref{#1}}
\def\ceil#1{\lceil #1 \rceil}
\def\floor#1{\lfloor #1 \rfloor}
\def\1{\bm{1}}

\def\vzero{{\bm{0}}}
\def\vone{{\bm{1}}}

\def\va{{\bm{a}}}
\def\vb{{\bm{b}}}
\def\vc{{\bm{c}}}

\def\ve{{\bm{e}}}
\def\vf{{\bm{f}}}
\def\vg{{\bm{g}}}
\def\vh{{\bm{h}}}

\def\vk{{\bm{k}}}

\def\vm{{\bm{m}}}

\def\vo{{\bm{o}}}
\def\vp{{\bm{p}}}
\def\vq{{\bm{q}}}

\def\vs{{\bm{s}}}
\def\vt{{\bm{t}}}
\def\vu{{\bm{u}}}
\def\vv{{\bm{v}}}

\def\vx{{\bm{x}}}
\def\vy{{\bm{y}}}
\def\vz{{\bm{z}}}

\def\mK{{\bm{K}}}

\def\mQ{{\bm{Q}}}

\def\mT{{\bm{T}}}

\def\mV{{\bm{V}}}
\def\mW{{\bm{W}}}

\DeclareMathAlphabet{\mathsfit}{\encodingdefault}{\sfdefault}{m}{sl}
\SetMathAlphabet{\mathsfit}{bold}{\encodingdefault}{\sfdefault}{bx}{n}

\def\gS{{\mathcal{S}}}

\def\sF{{\mathbb{F}}}

\def\sR{{\mathbb{R}}}

\def\sZ{{\mathbb{Z}}}

\def \ape {\text{ape}}

\def \rmd {\mathrm{d}}
\def \rme {\mathrm{e}}

\renewcommand{\mod}{\operatorname{mod}}

\theoremstyle{plain}
\newtheorem{theorem}{Theorem}[section]
\newtheorem{proposition}[theorem]{Proposition}
\newtheorem{lemma}[theorem]{Lemma}

\theoremstyle{definition}
\newtheorem{definition}[theorem]{Definition}
\newtheorem{assumption}[theorem]{Assumption}
\newtheorem{example}[theorem]{Example}
\newtheorem{remark}[theorem]{Remark}
\theoremstyle{remark}

\Crefname{assumption}{Assumption}{Assumptions}

\title{How Numerical Precision Affects Arithmetical Reasoning Capabilities of LLMs}

\newcommand*\samethanks[1][\value{footnote}]{\footnotemark[#1]}

\author{
 \textbf{Guhao Feng\thanks{Equal contribution.}\textsuperscript{1}},
 \textbf{Kai Yang\samethanks{}\textsuperscript{2}},
 \textbf{Yuntian Gu\textsuperscript{1}},
 \textbf{Xinyue Ai\textsuperscript{2}},
 \textbf{Shengjie Luo\textsuperscript{1}},\\
 \textbf{Jiacheng Sun\textsuperscript{3}},
 \textbf{Di He\textsuperscript{1}},
 \textbf{Zhenguo Li\textsuperscript{3}},
 \textbf{Liwei Wang\textsuperscript{1,4}}
\\
\\
 \textsuperscript{1}State Key Laboratory of General Artificial Intelligence, \\ School of Intelligence Science and Technology, Peking University \\
 \textsuperscript{2}School of EECS, Peking University \;
 \textsuperscript{3}Huawei Noah's Ark Lab \\
 \textsuperscript{4}Center for Machine Learning Research, Peking University
\\
 \small{
   \textbf{Correspondence:} \href{mailto:dihe@pku.edu.cn}{dihe@pku.edu.cn}
 }
}

\begin{document}
\maketitle
\begin{abstract}

Despite the remarkable success of Transformer-based large language models (LLMs) across various domains, understanding and enhancing their mathematical capabilities remains a significant challenge. In this paper, we conduct a rigorous theoretical analysis of LLMs' mathematical abilities, with a specific focus on their arithmetic performances. We identify numerical precision as a key factor that influences their effectiveness in arithmetical tasks. Our results show that Transformers operating with low numerical precision fail to address arithmetic tasks, such as iterated addition and integer multiplication, unless the model size grows super-polynomially with respect to the input length. In contrast, Transformers with standard numerical precision can efficiently handle these tasks with significantly smaller model sizes. We further support our theoretical findings through empirical experiments that explore the impact of varying numerical precision on arithmetic tasks, providing valuable insights for improving the mathematical reasoning capabilities of LLMs.

% Despite the success of Transformer-based Large Language Models (LLMs) in many areas, it still remains challenging to understand and improve their capabilities on mathematical reasoning. 
% In this paper, we aim to theoretically understand the importance of numerical precision in applying LLMs to mathematical tasks by focusing on three elementary arithmetic tasks: integer addition, iterative addition, and integer multiplication.
% We start by modeling low precision and normal precision as constant precision and logarithmic precision, respectively, which resembles the practical situations well.
% Our theoretical analysis shows that, although low precision Transformers can handle integer addition with reasonable model sizes, they fail to handle more complex arithmetic tasks, including iterative addition and integer multiplication.
% However, normal precision Transformers is capable for these elementary arithmetic tasks with reasonable model sizes. 
% Finally, we confirm our theoretical results thought experiments on different numerical precision and arithmetic tasks, providing guidance for better understanding and improvements of LLMs on mathematical reasoning. 

\end{abstract}

\section{Introduction}
\label{sec:intro}

\begin{table*}[t]
  \centering
  \renewcommand\arraystretch{1.5}
  \begin{tabular}{|c|c|c|}
    \hline
    \textbf{Arithmetic Tasks}           & \textbf{Standard Precision} & \textbf{Low Precision} \\
    \hline
    Integer Addition $\operatorname{ADD}_p(n)$       &     \cellcolor[RGB]{189,208,246} Constant     &     \cellcolor[RGB]{189,208,246}  $O(n^2)$                    \\
    \hline
    Iterated Addition $\operatorname{IterADD}_p(n, k)$     &     \cellcolor[RGB]{189,208,246} Constant        &    \cellcolor[RGB]{240,192,192}      Super-polynomial                 \\
    \hline
    Integer Multiplication $\operatorname{Mul}_p(n, l)$       &      \cellcolor[RGB]{189,208,246} $O(n^2)$       &  \cellcolor[RGB]{240,192,192}   Super-polynomial                      \\
    \hline
  \end{tabular}
  \caption{The model size \textit{w.r.t.} the input size required for various arithmetic tasks on bounded-depth Transformers, under both standard and low numerical precision.  Blue denotes the acceptable model size, and red represents the unaffordable model size.}
  \label{table:thm-result}
\end{table*}

Transformer-based LLMs, such as GPT \cite{openai2023gpt4}, Claude \cite{TheC3}, and LLAMA \cite{dubey2024llama}, have achieved impressive performance across a broad range of natural language tasks \cite{basyal2023textsummarizationusinglarge,shao-etal-2023-character,zhu-etal-2024-multilingual}. Despite the great success, significant challenges remain when applying LLMs to mathematical problem-solving. Unlike many typical NLP tasks, which often depend on pattern recognition and statistical correlations \cite{blei2003latent}, mathematical reasoning requires rigorous logical deduction in a specific order 
\citep{bubeck2023sparks,frieder2024mathematical}.
To address these challenges, various strategies have been proposed, including carefully designed prompting
strategies \citep{wei2022chain,yamauchi2023lpmlllmpromptingmarkuplanguage,imani-etal-2023-mathprompter} and inference-based searching method \citep{kang2024mindstarenhancingmathreasoning,wu2024empirical,snell2024scaling,brown2024large}. However, a comprehensive understanding of the intrinsic limitations that restrict the mathematical reasoning capabilities of LLMs remains elusive. 

In principle, mathematical reasoning, built on basic arithmetical operations, requires accurate computation of intermediate results throughout the reasoning process \citep{bubeck2023sparks,lee2024teaching}. There exist works \citep{feng2023towards,yang2024do} exploring the arithmetic capabilities of LLMs with Chain of Thought (CoT) prompting \citep{wei2022chain}. However, these investigations often deviate from the tokenization strategies employed by modern LLMs \citep{openai2023gpt4,dubey2024llama}, where numbers are typically segmented into tokens of at most three digits. Under the assumption of \citet{feng2023towards} and \citet{yang2024do}, each distinct number occupies a unique position in the vocabulary, leading to an essential mismatch with practical implementations. Moreover, recent studies have demonstrated that LLMs operating with reduced numerical precision (e.g., \texttt{int4}) exhibit a significant decline in performance on mathematical tasks \citep{jin2024comprehensive,marchisio2024does}.

In this paper, we provide a rigorous theoretical investigation of the arithmetical abilities of LLMs under the autoregressive paradigm. Specifically, we adopt the tokenization approach used in modern LLMs, where numbers are processed and generated token by token, with each token representing only a small number of digits. Under these assumptions, we identify \textbf{numerical precision} as a key factor influencing their performance in arithmetical tasks. Our analysis focuses on three elementary arithmetic tasks: integer addition, iterated addition, and integer multiplication, which serve as elementary building blocks in solving complex real-world math problems.

To elucidate the role of numerical precision, we first examine the expressiveness of Transformers operating under low precision, such as \texttt{int8} and \texttt{int4}. We establish foundational impossibility results for low-precision Transformers, demonstrating that such models require super-polynomial size with respect to input length to solve iterated addition and integer multiplication (\cref{thm:con-iter-add-con_dep,thm:con-mul-mod-con_dep}). Our proofs, grounded in complexity theory \citep{razborov1987lower,arora2009computational}, show that this limitation arises from the inability of individual neurons to store intermediate results during arithmetic computations. As a result, a significantly larger number of neurons is required to distribute the computation and avoid overflow.

We further demonstrate that increasing numerical precision is essential to addressing this limitation. Specifically, as numerical precision improves, the model size required to solve arithmetic tasks decreases significantly. In particular, we prove that a bounded-depth Transformer operating with standard precision (e.g., \texttt{float32}) can efficiently and reliably solve all three tasks under consideration. For both integer and iterated addition, the required model size remains constant and independent of the input length (\cref{thm:log-int-add-con,thm:log-iter-add-con}), while for integer multiplication, the model size scales quadratically w.r.t the input length (\cref{thm:log-mul-mod-con}). These results highlight that standard numerical precision is sufficient for LLMs to effectively perform arithmetic tasks. Our findings emphasize the practical importance of numerical precision in mathematical reasoning. While low-precision models may offer computational advantages, ensuring sufficient numerical precision is critical for tasks involving complex arithmetic. A summary of our main results is provided in \Cref{table:thm-result}.

In addition to theoretical analysis, we conduct extensive experiments to validate our conclusions. First, we evaluate the performance of Transformers trained from scratch on the aforementioned arithmetic tasks, systematically examining how problem size and numerical precision impact their capabilities. Furthermore, we also conduct experiments on LLAMA-3.1-8B Instruct \citep{dubey2024llama} to evaluate the performance of these arithmetic tasks under different numerical precision. Our empirical results demonstrate that both low-precision and standard-precision Transformers perform adequately on the integer addition task. However, as task complexity increases—particularly in iterated addition and integer multiplication—the decrease in precision in Transformers results in significant performance degradation. These findings align with our theoretical predictions and offer practical guidance for enhancing LLM performance in mathematical reasoning tasks.

\section{Preliminary}
\label{sec:pre}

An autoregressive Transformer, or decoder-only Transformer \citep{radford2019language, dai2019transformer}, is a neural network designed to model sequence-to-sequence mappings. For an input sequence $\vs$ of length $n$, each input token $s_i$ (for $i \in [n]$) is transformed into a $d$-dimensional vector $\vx_i^{(0)} = \text{Embed}(s_i) + \vp_i \in \mathbb{R}^d$, where $\text{Embed}(\cdot)$ represents the token embedding function, and $\vp_i$ denotes learnable positional embeddings. The model then consists of $L$ Transformer blocks, each following the form:
\begin{gather*}
    \vh^{(l)}_i = \vx^{(l-1)}_i + \mathrm{Attn}^{(l)} \left( \vx^{(l-1)}_i; \{ \vx^{(l-1)}_j : j \leq i \} \right), \\
    \vx^{(l)}_i = \vh^{(l)}_i + \mathrm{FFN}^{(l)}(\vh^{(l)}_i),
\end{gather*}
where $l \in [L]$. Here, $\mathrm{Attn}^{(l)}$ and $\mathrm{FFN}^{(l)}$ denote the multi-head self-attention layer and the feed-forward network of the $l$-th Transformer block:
\begin{equation*}
    \begin{aligned}
   & \mathrm{Attn}^{(l)}(\vx, \gS)=\sum_{h=1}^H\left(\mW_\text{O}^{(l,h)}\right)^\top \cdot \mathrm{H}^{(l, h)}(\vx, \gS), \\
   &\mathrm{H}^{(l, h)}(\vx, \gS) =  \\
   & \mathrm{softmax}_{\vz \in \gS}\left((\mW_\text{K}^{(l,h)}\vz)^\top(\mW_\text{Q}^{(l,h)}\vx)\right) \mW_\text{V}^{(l,h)}\vz, \\
   & \mathrm{FFN}^{(l)}(\vx)=\mW_2^{(l)}\sigma(\mW_1^{(l)}\vx), 
\end{aligned}
\end{equation*}
where $\mW_\text{Q}^{(l,h)}$,$\mW_\text{K}^{(l,h)}$,$\mW_\text{V}^{(l,h)}$,$\mW_\text{O}^{(l,h)}$ $\in \mathbb R^{\lceil \frac d H\rceil\times d}$ are the query, key, value, and output matrices of the $h$-th head in the $l$-th layer. The weight matrices in the feed-forward network are denoted as $\mW_1^{(l)},\mW_2^{(l)}\in \mathbb R^{d \times d}$. The activation function $\sigma$ is chosen to be GeLU \citep{hendrycks2016gaussian}, following the work of \cite{radford2019language}. 

The computed embedding $\vx_n^{(M)}$ is then used to predict the next token $s_{n+1}$, which is concatenated to the input to continue the sequence generation process. This process terminates when an \texttt{<EOS>} token is generated.
Further discussions on related work are listed in \cref{sec:related}.

\section{Problem Setup}
\label{sec:problem}

This paper explores the arithmetic reasoning capabilities of LLMs by focusing on three elementary arithmetic tasks: integer addition, iterated addition, and integer multiplication under the autoregressive paradigm. Below, we define the integer representations used throughout the study and provide formal descriptions for each task.

\begin{figure*}[t!]
    \centering
    \includegraphics[width=0.8\linewidth]{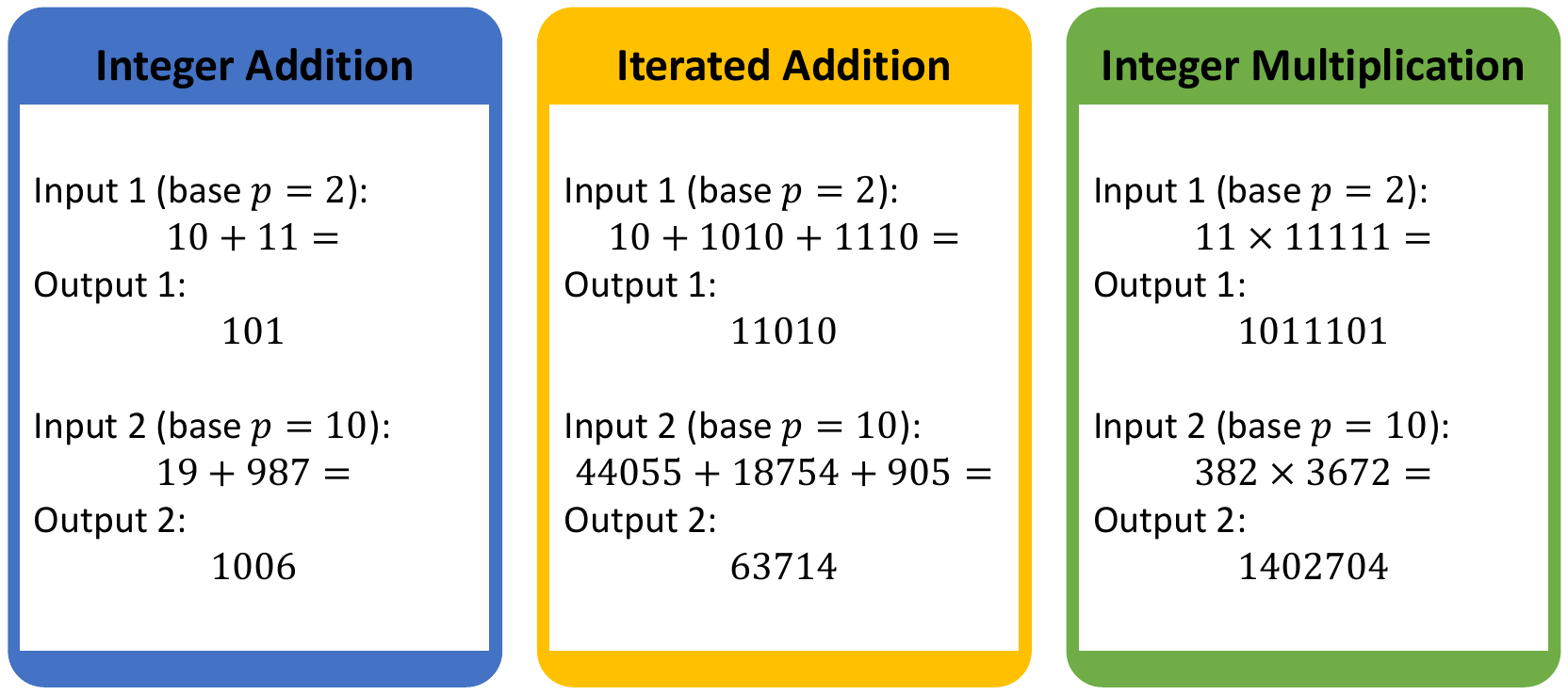}
    \caption{Examples for three elementary arithmetic tasks we consider in this paper: integer addition, iterated addition, and integer multiplication.}
    \label{fig:arithmetic-task}
\end{figure*}

\textbf{Integer Representation and Tokenization. } 
We consider all integers to be non-negative and represented in base-$p$ notation, where $p \geq 2$ is a fixed base. Specifically, an integer with $n$ digits is expressed as $(x_{n-1} \cdots x_0)_p$. To tokenize this sequence, we employ a tokenizer, denoted by $\mT_c$, that partitions $\mathbf{x}$ into tokens, each containing at most $c$ contiguous digits. Formally, let the sequence $\vt=[t_{k-1},\ldots,t_0]=\mT_c([x_{n-1}, \ldots, x_0])$, where $k=\ceil{\frac{n}{c}}$ we have 
\begin{equation}
\label{eq:tokenizer}
    \begin{aligned}
        t_i=\begin{cases}
            [x_{ic},x_{ic+1},\cdots,x_{ic+c-1}], & i < k-1;\\
            [x_{ic},x_{ic+1},\cdots,x_{n-1}], & i=k-1.\\
        \end{cases}
    \end{aligned}
\end{equation}
During sequence generation, the Transformer model outputs the target tokens sequentially, strictly following the tokenization scheme of tokenizer $\mT_c$. Unlike prior works \citep{feng2023towards,yang2024do}, which represent entire integers as single tokens, this approach aligns with prevalent tokenization strategies employed by modern LLMs \citep{openai2023gpt4,dubey2024llama} and enables Transformers to process and generate numbers token by token. Further discussion and illustrative examples of the tokenization scheme are provided in \cref{sec:tokenization}.

\textbf{Integer Addition.} 
Let $\va = (a_{n_1 - 1} \cdots a_0)_p$ and $\vb = (b_{n_2 - 1} \cdots b_0)_p$ denote two integers represented in base-$p$. Their sum is expressed as $\vs = (s_n \cdots s_0)_p = \va + \vb$. Let $\mT_c$ represent the tokenizer. The input sequence is constructed by concatenating the tokenized representations of $\va$ and $\vb$, i.e., $\mT_c(\va)$ and $\mT_c(\vb)$, with the addition operator token `$+$' placed between them, and the equality operator token `$=$' appended at the end. The task is to generate the tokenized representation of the result, $\mT_c(\vs)$, sequentially, one token at a time.

\textbf{Iterated Addition.} 
Now consider $k$ integers in base-$p$: $\va_1 = (a_{1,n_1-1} \cdots a_{1,0})_p, \; \ldots, \; \va_k = (a_{k,n_k-1} \cdots a_{k,0})_p$, where $n = \max\{n_1, \ldots, n_k\}$. Their sum is denoted as $\vs = (s_{n-1} \cdots s_0)_p = \sum_{i \in [k]} \va_i$, where $n=\max_{i\in[k]}\{n_k\}+\ceil{\log k}$. Let $\mT_c$ denote the tokenizer. The input sequence is formed by concatenating the tokenized representations of these integers, separated by the addition operator token `$+$', followed by the equality operator token `$=$' appended at the end. The objective is for the Transformer to generate the tokenized representation of the sum, $\mT_c(\vs)$, sequentially, one token at a time.

\textbf{Integer Multiplication.} 
The integer multiplication task involves computing the product of two integers, truncated to a predefined length $l$. Let $\va = (a_{n_1 - 1} \cdots a_0)_p$ and $\vb = (b_{n_2 - 1} \cdots b_0)_p$ represent two integers in base-$p$, and let $n = \max\{n_1, n_2\}$. Their product is given by $\vs = (s_{2n-1} \cdots s_0)_p = \va \times \vb$. Let $\mT_c$ denote the tokenizer. The input sequence is constructed by concatenating the tokenized representations of $\va$ and $\vb$, separated by the multiplication operator token `$\times$', with the equality operator token `$=$' appended at the end. The objective is to generate the tokenized representation of the product's least significant $l$ digits, $\mT_c([s_{l-1}, s_{l-2}, \ldots, s_0])$, where $l \leq 2n$.

\begin{remark}
    We consider a generalized case of integer multiplication where overflow may occur if the result exceeds the given digit length. Standard integer multiplication is a special case of this framework when $l = n_1+n_2$.
\end{remark}

\Cref{fig:arithmetic-task} presents examples of these tasks. Integer addition is the simplest of these tasks and can be viewed as a specific instance of iterated addition. Furthermore, integer multiplication inherently involves the summation of several intermediate products. Consequently, we present these tasks in increasing order of complexity. In the subsequent sections, we use the notations $\operatorname{ADD}_p(n)$ to denote addition with at most $n$ digits in base-$p$ arithmetic, $\operatorname{IterADD}_p(n, k)$ for the iterated addition of $k$ integers with at most $n$ digits each in base-$p$, and $\operatorname{MUL}_p(n, l)$ for the multiplication of two integers with at most $n$ digits in base-$p$, truncated to $l$ digits.

\section{Low-Precision Transformers Struggle with Basic Arithmetic Tasks}
\label{sec:constant-precision}

Recent studies \citep{marchisio2024does,jin2024comprehensive} have shown that LLMs operating under low-precision constraints encounter significant challenges in performing basic mathematical tasks. In this section, we examine the expressive limitations of Transformers under such constraints and seek to explain the sharp decline in their arithmetical capabilities. Specifically, we demonstrate that Transformers restricted to low-precision arithmetic exhibit substantial difficulty in solving even elementary arithmetic problems. 

To formalize these limitations, we build on the framework introduced by \citet{li2024chain} and utilize the setting of a \textbf{constant-precision Transformer} (See formal definition in \cref{sec:def_constant_precision}). In this setting, the internal states of the model's neurons are constrained to represent real numbers using only $c$ bits, where $c$ is a small constant independent of the input sequence length. These numbers may be represented by floating point in IEEE 754 formats \citep{kahan1996ieee} or fixed point formats. This configuration mirrors many practical deployment scenarios, in which LLMs often employ reduced-precision formats such as \texttt{float8}, \texttt{int8}, or even \texttt{int4}, particularly during inference \cite{han2015deep}. Given that these models typically process input sequences comprising thousands of tokens, it is reasonable and realistic to assume that the numerical precision remains fixed at a small constant, independent of sequence length. Under the constant-precision setting, we examine the expressiveness of the Transformer model in elementary arithmetic problems.

\begin{theorem}
\label{thm:con-int-add-lower}
Fix integers $p\geq 2$ and $c\in \mathbb{N}^*$. Consider the tokenizer $\mT_c$ defined in \cref{eq:tokenizer} for processing the input and output sequences. There exist constant-precision Transformers with constant depth (independent of $n$) and hidden dimension $d = O(n^2)$ that can solve the $\operatorname{ADD}_p(n)$ task. 

\end{theorem}

\cref{thm:con-int-add-lower} suggests that the bounded-depth Transformers with reasonable hidden dimensions are capable of solving the integer addition task. However, as we will show in subsequent theorems, constant-precision Transformers exhibit pronounced limitations when considering more complex arithmetic problems. For the page limitation, we give the detailed proof of \cref{thm:con-int-add-lower} in \cref{sec:proof:con-int-add-lower}.
\begin{theorem}
\label{thm:con-iter-add-con_dep}
Fix integers $p\geq 2$ and $c,L\in \mathbb{N}^*$. Consider the tokenizer $\mT_c$ defined in \cref{eq:tokenizer} for processing the input and output sequences. For any polynomial $f$, there exist problem scales $n$ and $k$ such that no constant-precision autoregressive Transformer with $L$ layers and hidden dimension $d < f(n, k)$ can correctly solve the $\operatorname{IterADD}_p(n,k)$ task.
\end{theorem}

\begin{theorem}
\label{thm:con-mul-mod-con_dep}
Fix integers $p\geq 2$ and $c,L\in \mathbb{N}^*$. Consider the tokenizer $\mT_c$ defined in \cref{eq:tokenizer} for processing the input and output sequences. For any polynomial $f$, there exist problem scales $n$ and $l$ such that no constant-precision autoregressive Transformer with $L$ layers and hidden dimension $d < f(n, l)$ can correctly solve the $\operatorname{MUL}_p(n, l)$ task.
\end{theorem}

The detailed proof of \Cref{thm:con-iter-add-con_dep,thm:con-mul-mod-con_dep} are presented in \Cref{sec:proof_con-iter-add-con_dep,sec:proof:con-mul-mod-con_dep}.

\textbf{What accounts for this limitation?} As presented in \Cref{app:proof-constant}, our proof is grounded in circuit complexity theory. By modeling the constant-precision Transformer as a computational circuit, we rigorously analyze its expressive limitations through the lens of circuit complexity \citep{merrill2022saturated,merrill2023parallelism,feng2023towards,li2024chain}. Specifically, \citet{li2024chain} proves that the expressiveness of constant-precision Transformers with polynomial size and bounded depth is upper-bounded by the computation complexity class $\mathsf{AC}^0$. In contrast, we demonstrate that the complexity of tasks such as
$\operatorname{IterADD}$ and $\operatorname{MUL}$ exceeds that of $\mathsf{AC}^0$, using reductions from $\texttt{Majority}$, a well-established problem that has been provably unsolvable by the circuits in $\mathsf{AC}^0$ \citep{razborov1987lower,smolensky1987algebraic}. Consequently, these tasks are inherently hard for low-precision Transformers.

\textbf{Practical Implications.} While low-precision Transformers can effectively handle some of the simplest arithmetic tasks, such as basic integer addition, their capacity is severely limited when addressing more complex tasks. As demonstrated, low numerical precision, such as \texttt{int4} and \texttt{float8}, imposes fundamental constraints, preventing these models from solving problems that would require Transformers with super-polynomial size.

\section{Standard-Precision Transformers Are Sufficient for Arithmetic Tasks}
\label{sec:log-precision}
In \Cref{sec:constant-precision}, we demonstrated that low-precision Transformers struggle with arithmetic tasks due to their expressive limitations. In this section, we will show that increasing numerical precision is essential to overcoming this limitation. In particular, we focus on \textbf{standard-precision} Transformers and show that such models can overcome these limitations and solve arithmetic problems efficiently.

To formalize the notion of standard precision (e.g., \texttt{float32}), we follow \citet{feng2023towards} and adopt the setting of a \textbf{logarithmic-precision Transformer} (See formal definition in \cref{app:bg}). In this setting, the Transformer’s internal neurons can represent real numbers with up to $O(\log n)$ bits, where $n$ denotes the maximum input sequence length. 
Given that modern LLMs often limit their context length to hundreds of thousands of tokens \citep{openai2023gpt4,touvron2023llama2,TheC3}, it is natural to treat $32$ as the logarithmic scale corresponding to $100,000$. Hence, the logarithmic-precision setting reflects practical deployment scenarios.

We first establish that, under logarithmic precision, a Transformer with constant depth and dimension can solve both the integer addition and iterated addition tasks for arbitrarily large input lengths, as shown in \Cref{thm:log-int-add-con,thm:log-iter-add-con}. The detailed proof of \cref{thm:log-int-add-con,thm:log-iter-add-con} is presented in \cref{sec:proof_log-int-add-con,sec:proof_log-iter-add-con}.

\begin{theorem}
\label{thm:log-int-add-con}
Fix integers $p\geq 2$ and $c\in \mathbb{N}^*$. Consider the tokenizer $\mT_c$ defined in \cref{eq:tokenizer} for processing the input and output sequences. There exists a logarithmic-precision Transformer with constant depth and hidden dimension (independent of $n$) that can generate the correct output for any input on the $\operatorname{ADD}_p(n)$ task.
\end{theorem}

\begin{theorem}
\label{thm:log-iter-add-con}
Fix integers $p\geq 2$ and $c\in \mathbb{N}^*$. Consider the tokenizer $\mT_c$ defined in \cref{eq:tokenizer} for processing the input and output sequences. For any integers $n$ and $k$, there exists a logarithmic-precision Transformer with constant depth and hidden dimension $d$ (independent of $n$ and $k$) that can generate the correct output for any input on the $\operatorname{IterADD}_p(n, k)$ task.
\end{theorem}

We now turn to integer multiplication. As established in \Cref{thm:log-mul-mod-con}, a logarithmic-precision Transformer with constant depth and polynomial hidden dimensions is capable of solving the integer multiplication task. The detailed proof of \cref{thm:log-mul-mod-con} is presented in \cref{sec:proof-log-mul-mod-con}.

\begin{theorem}
\label{thm:log-mul-mod-con}
Fix integers $p\geq 2$ and $c\in \mathbb{N}^*$. Consider the tokenizer $\mT_c$ defined in \cref{eq:tokenizer} for processing the input and output sequences. For any integers $n$ and $l\leq 2n$, there exists a logarithmic-precision Transformer with constant depth (independent of $n$ and $k$) and hidden dimensions $O(n^2)$ that can generate the correct output for any input on the $\operatorname{MUL}_p(n, l)$ task.
\end{theorem}

\Cref{thm:log-int-add-con,thm:log-iter-add-con,thm:log-mul-mod-con} demonstrate that, under standard precision, a bounded-depth Transformer with reasonable size can solve all elementary arithmetic tasks. Compared to the theoretical results for low-precision Transformers (\cref{thm:con-int-add-lower,thm:con-iter-add-con_dep,thm:con-mul-mod-con_dep}), even a modest increase in numerical precision leads to a substantial improvement in expressiveness for arithmetic tasks.

\textbf{The Reason for Increased Expressiveness.} The transition from constant precision to logarithmic precision enables Transformers to process and represent large numbers effectively, thereby expanding their expressiveness beyond the capabilities of low-precision models. In particular, the expressiveness of a logarithmic-precision Transformer with polynomial size and bounded depth is upper-bounded by the computational complexity class $\mathsf{TC}^0$ \citep{merrill2023parallelism}. Leveraging this increased precision, we constructively prove that logarithmic-precision Transformers are sufficient for solving these arithmetic tasks. These results underscore the critical role of numerical precision in enhancing the expressiveness of Transformer architectures.

\textbf{Practical Implications.} Our theoretical results underscore the critical importance of numerical precision when deploying Transformers for arithmetic tasks. Under low-precision settings, a Transformer requires super-polynomial model size to solve even elementary arithmetic problems, which is impractical for real-world applications. While low-precision models may offer computational efficiency, they are likely to fail in scenarios that demand accurate numerical reasoning, such as mathematical problem-solving or scientific computing. However, a slight increase in precision—such as using \texttt{float32}—enables Transformers to handle more complex arithmetic operations while maintaining a reasonable hidden dimension. Thus, employing sufficient numerical precision is crucial for ensuring both accuracy and robustness in arithmetic tasks, and should be a key consideration when designing or deploying LLMs for applications involving complex arithmetic reasoning.
\begin{figure*}[!t]
    \centering
    \includegraphics[width=\linewidth]{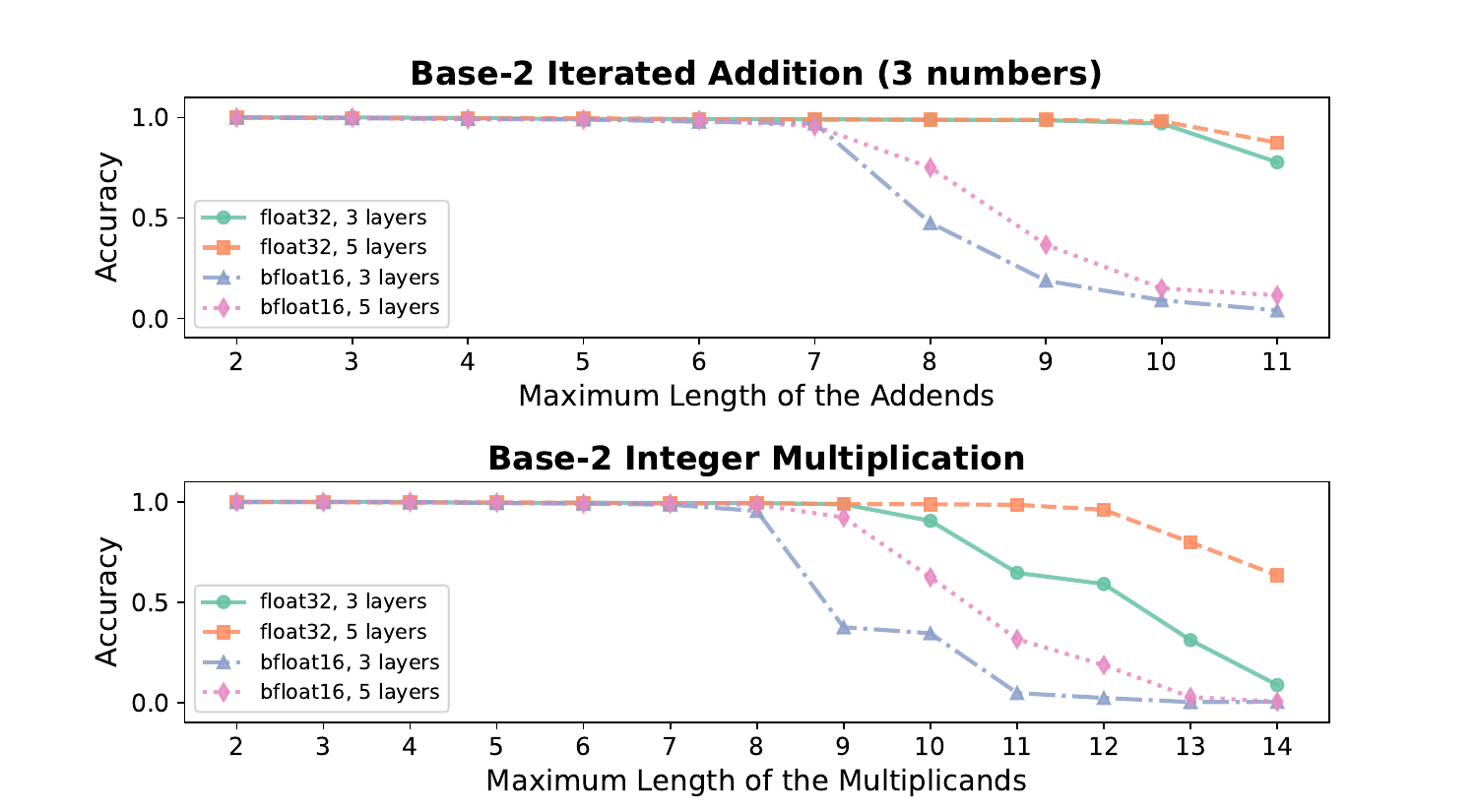}
    \caption{Model performance on different tasks in base-2. Within each sub-figure, the x-axis represents the maximum digits length and the y-axis represents the accuracy gained by each model. The figure indicates that, for all tasks, Transformers utilizing \texttt{float32} with 3 layers and 5 layers outperform their \texttt{bfloat16} counterparts.}
    \label{fig:experiments-results}
\end{figure*}

\section{Experiments}
\label{sec:exp}

In the preceding sections, we employ complexity theory to demonstrate that low-precision Transformers face significant challenges in performing elementary arithmetic tasks. To validate these theoretical insights, we conduct a series of experiments to compare the performance of Transformers under different precisions. The results provide empirical evidence that the model's ability to execute arithmetic operations drops as precision decreases, reinforcing our theoretical results.

\subsection{Experimental Setup}

\textbf{Tasks and datasets.} We evaluate three elementary arithmetic tasks: integer addition, iterated addition, and integer multiplication, as presented in \Cref{fig:arithmetic-task}. Each task involves a series of experiments with base $p = 2, 10$ and varying choices of digit length $n$. For integer addition, we examine the addition of integers in both base-2 and base-10, with digit lengths $n \in \{4, 8, 16, 32, 64\}$. For iterated addition, we examine the addition of three numbers in base-2, with digit lengths $n \in [2, 11]$, as well as in base-10, with digit lengths $n \in [1, 4]$. Similarly, for integer multiplication, we run experiments in base-2 with digit lengths $n \in [2, 14]$, and in base-10 with digit length $n \in [2, 5]$. Both training data and test data are dynamically generated. We use a batch size of 512 with 100k steps, resulting in a total training dataset size of 51.2M. Further details regarding the data generation function and the construction of datasets are provided in \Cref{alg:itadd-implement,alg:mult-implement}.

\textbf{Training and Evaluation.} All experiments use Transformers as the backbone. We trained models with 3 and 5 layers and evaluated their performance on each task. Detailed model and training configurations are listed in \cref{tab:model_config,tab:training_config}. No prompts or chat templates were added to the dataset. The models were trained with cross-entropy loss over the answer tokens. During evaluation, the models were required to produce exact answers, with accuracy reported as the evaluation metric. For each task, accuracy was computed over 50k test samples. To assess the impact of numerical precision, experiments were conducted with \texttt{float32} and \texttt{bfloat16}.

\begin{figure*}[!t]
    \centering
    \includegraphics[width=\linewidth]{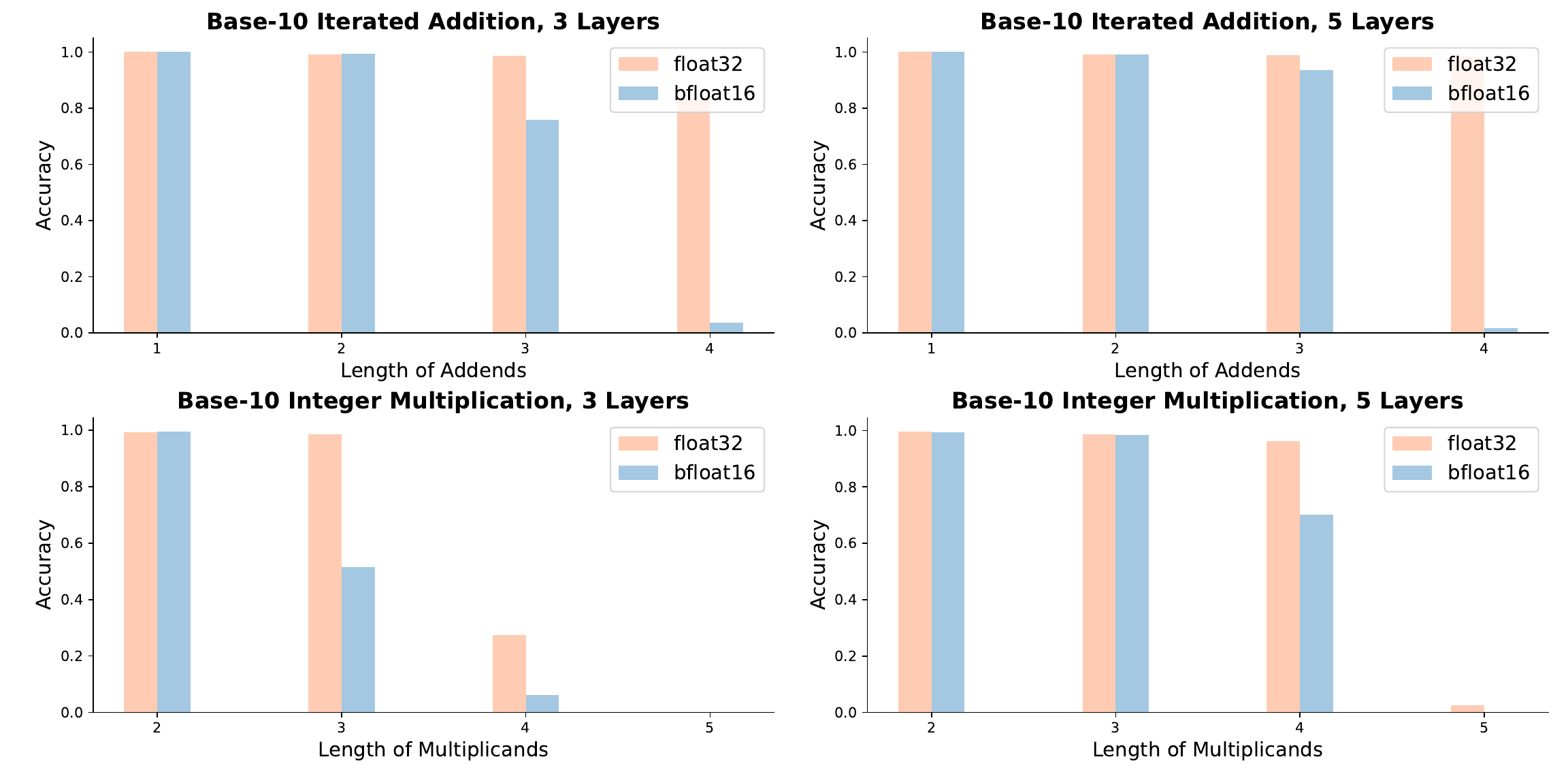}
    \caption{Model performance on iterated addition tasks involving three numbers and integer multiplication tasks. Each sub-figure presents a comparison of the performance between \texttt{float32} and \texttt{bfloat16}.}
    \label{fig:base10-results}
\end{figure*}

\subsection{Experimental Results}

Integer addition proved relatively simple, maintaining over 94\% accuracy even as digit lengths increased to 32 across both base-2 and base-10 for both \texttt{float32} and \texttt{bfloat16} (see \Cref{sec:int_add_result}). 

The results for iterated addition and multiplication in base-2 are shown in \Cref{fig:experiments-results}, while the corresponding base-10 results are presented in \Cref{fig:base10-results}. In each sub-figure, the x-axis represents the maximum digit length for addends or multiplicands, while the y-axis indicates test accuracy. 

For iterated addition, accuracy under \texttt{bfloat16} declined significantly as the digit length increased, while \texttt{float32} consistently achieved near-perfect accuracy across all model depths. Specifically, in base-2, 16-bit precision exhibited a pronounced decline for digit lengths between 7 and 10, whereas 32-bit precision maintained high accuracy. In base-10, at digit lengths up to 10, \texttt{float32} achieved over 90\% accuracy, whereas \texttt{bfloat16} struggled to produce correct results.

In the multiplication task, the gap between the two precisions became even more apparent as digit lengths increased. For example, at a digit length of 13 in base-2, 16-bit precision accuracy dropped sharply, signifying its inability to handle such inputs. Similarly, in base-10, 16-bit precision showed a marked reduction in accuracy, particularly for inputs with lengths of 3 in 3-layer models and lengths of 4 in 5-layer models. These results underscore the critical role of precision in achieving reliable performance for elementary arithmetic tasks, consistent with our theoretical findings.

\subsection{Further Experiments on LLMs}

\begin{figure*}[!t]
    \centering
    \includegraphics[width=\linewidth]{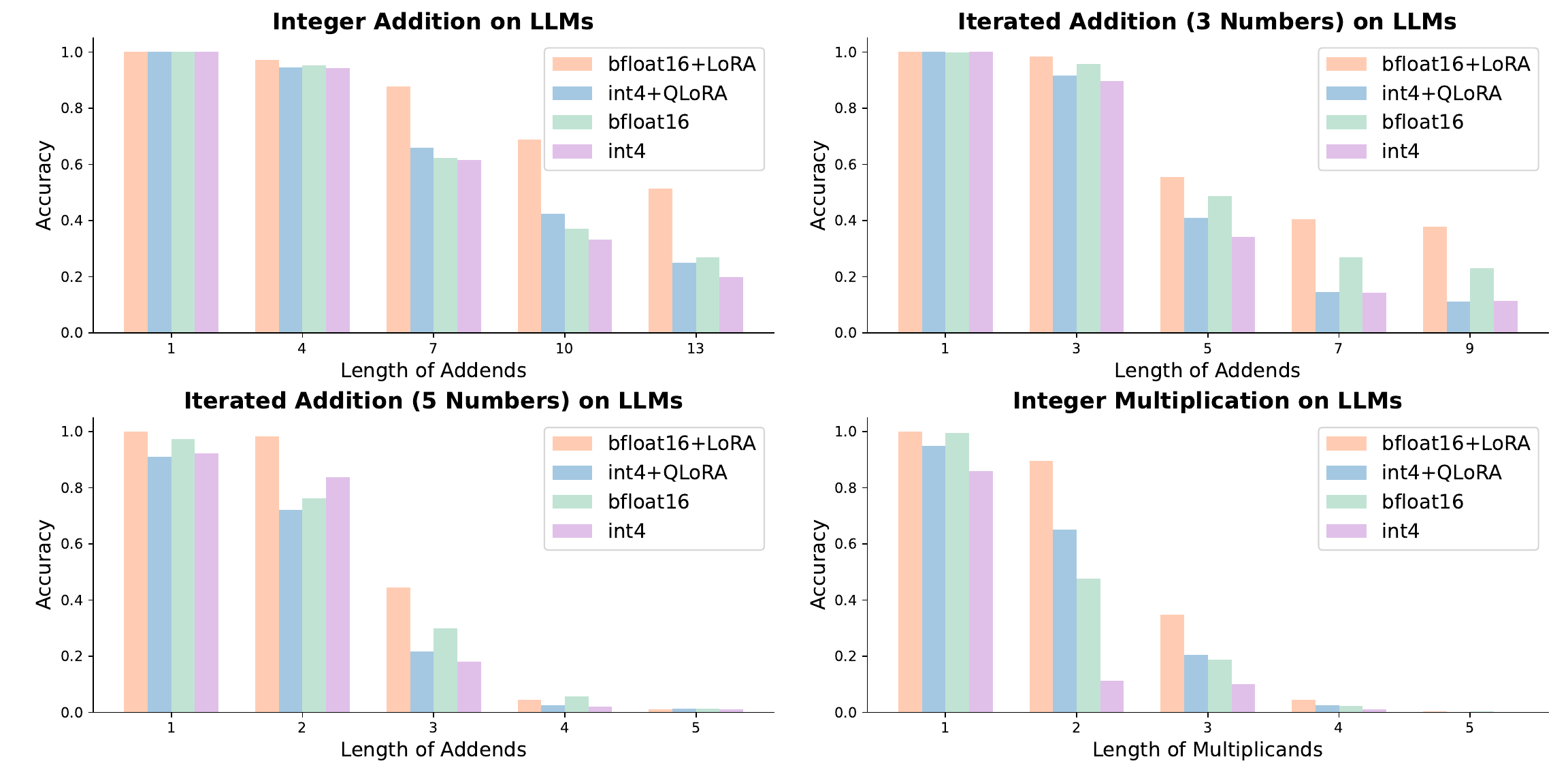}
    \caption{The performance of LLAMA-3.1-8B Instruct model on arithmetic tasks in base-10. In each sub-figure, we compare the original model in \texttt{bfloat16} and the quantized model in \texttt{int4}, alongside fine-tuned models, with LoRA using \texttt{bfloat16} and QLoRA using \texttt{int4}.}
    \label{fig:exp:QLoRA}
\end{figure*}

To further substantiate our theoretical results, we conducted additional experiments on LLMs, specifically evaluating the LLAMA-3.1-8B Instruct model on elementary arithmetic tasks.

\textbf{Task Description.} For integer addition, we tested the addition of two base-10 integers with digit lengths ranging from $1$ to $13$. For iterated addition, we extended the task to include three and five base-10 numbers, with digit lengths spanning $1$ to $9$ and $1$ to $5$, respectively. For integer multiplication, we evaluated the multiplication of two base-10 numbers, with digit lengths varying from $1$ to $5$. Data generation followed the same procedure as earlier experiments, with details provided in \Cref{alg:itadd-implement,alg:mult-implement}.

\textbf{Model Configuration.} All experiments used the LLAMA-3.1-8B Instruct model \citep{dubey2024llama}. To study the effects of reduced precision, we evaluated the model under four settings:
\begin{itemize} \setlength\itemsep{-3pt}
    \item Original model operating under \texttt{bfloat16}
    \item Quantized model operating under \texttt{int4}
    \item Fine-tuned model using LoRA (\texttt{bfloat16})
    \item Fine-tuned model using QLoRA (\texttt{int4})
\end{itemize}
The baseline configuration employs the original LLaMA-3.1-8B Instruct model, which operates under \texttt{bfloat16} precision. To assess the effects of reduced precision, we applied 4-bit quantization using the AWQ algorithm \citep{lin2024awq}. Further, we fine-tuned the model using LoRA and QLoRA \citep{hu2021lora,dettmers2024qlora}. The fine-tuning configurations for LoRA and QLoRA are listed in \cref{tab:lora_config} in \cref{sec:detail-llm}. For the LoRA fine-tuning experiments, model weights were maintained in \texttt{bfloat16}. In contrast, the QLoRA experiments extended this setup by enabling 4-bit quantization, represented by the \texttt{int4}. Fine-tuning was performed individually for each task. Furthermore, we add a baseline of GPT-4o \citep{openai2023gpt4} as a reference, whose results are listed in \cref{tab:gpt4}.

\textbf{Dataset for Fine-tuning.} We generated the fine-tuning data for both multiplication and addition tasks, with both multiplicands and addends varying in length from 1 to 9. The dataset comprised a total of 60k samples, including 5k samples with lengths between 1 and 6, and 10k samples with lengths between 7 and 9. The generation process of the dataset is the same as in previous experiments. Furthermore, we add the few-shot learning prompt to the raw dataset and apply the LLaMA chat template for data preprocessing. The prompt for few-shot learning can be found in \cref{tab:prompt_add,tab:prompt_mul}.% in \Cref{sec:detail-llm}.

\textbf{Evaluation.} For the evaluation, we employ a few-shot learning approach for inference. The prompts are the same as the prompts of fine-tuning dataset and can be found in \cref{sec:gen_data}. The generation configurations for LLMs can be also found in \Cref{tab:generation_config} in \cref{sec:detail-llm}. During inference, the LLMs were tasked with producing exact solutions to the given arithmetic problems. Both the original model and the model fine-tuned with LoRA are evaluated using $\texttt{bfloat16}$, whereas the quantized model and the model fine-tuned with QLoRA are evaluated using $\texttt{int4}$. For each task, we evaluate the model on 1k samples to compute the accuracy serving as the evaluation metric.

The results of the experiments are shown in \Cref{fig:exp:QLoRA}. Each sub-figure presents the results of a task, where the x-axis denotes the maximum length of the addends or multiplicands, and the y-axis represents the test accuracy. For each task, reducing numerical precision in both the original and fine-tuned models leads to a significant decrease in accuracy. Specifically, in the iterated addition task for 3 numbers, accuracy drops by nearly $20\%$ as the length of the addends increases. Similarly, for models fine-tuned with QLoRA and LoRA, lowering precision also results in a decline in accuracy. Furthermore, in some cases, even after fine-tuning a low-precision model with QLoRA, the performance does not surpass that of the original model with standard precision. These experimental findings support our theoretical results that numerical precision is a critical factor in the success of iterated addition and integer multiplication tasks. Overall, the results underscore the consistency between the precision requirements for these elementary arithmetic tasks and our theoretical predictions.

\section{Conclusion}
\label{sec:con}

In this work, we have theoretically analyzed the impact of numerical precision on LLMs for arithmetical reasoning. By focusing on three elementary arithmetic tasks, integer addition, iterated addition, and integer multiplication, we demonstrate that the Transformers operating under standard precision can handle these tasks effectively. In contrast, Transformers with low precision struggle with complex arithmetic tasks, excelling only at integer addition. 
Extensive experimental results corroborate our theoretical findings, showing that standard precision models outperform low precision ones. 
We believe this study offers valuable insights for developing more powerful LLMs in mathematics.

\section{Limitations}

One limitation of this work is that we have not fully explored all key components of mathematical reasoning. While the arithmetic tasks considered are foundational, there remain other essential elements of mathematical reasoning whose dependence on numerical precision is still unclear. 
Additionally, our focus was exclusively on numerical precision, but we acknowledge that other factors are likely to play a significant role in applying LLMs to mathematical reasoning. 
We leave these explorations for future work.

\section*{Acknowledgments}
Liwei Wang is supported by National Science Foundation of China (NSFC62276005). Di He is supported by National Science Foundation of China (NSFC62376007). This work is partially supported by the Shanghai Committee of Science and Technology (Grant No. 21DZ1100100).

\bibliography{custom}

\newpage
\onecolumn

\appendix

\section{Related Work}
\label{sec:related}

\subsection{LLMs for Mathematical Reasoning}

\paragraph{Mathmetical Reasoning.} Recent studies highlight the limitations of current LLMs in mathematical reasoning \cite{ahn-etal-2024-large,srivastava-etal-2024-evaluating}. \citet{10.1145/3626772.3657945} demonstrated that advanced models like GPT-4 can generate relevant answers, but these answers are not always accurate. Additionally, \citet{mao2024champcompetitionleveldatasetfinegrained} found that current LLMs struggle even with verifying the solutions to mathematical problems. To enhance the mathematical capabilities of LLMs, several studies have carefully designed prompting strategies \cite{shakarian2023independentevaluationchatgptmathematical,cheng-yu-2023-analyzing,gu2023llmspotentialbrainstormingpartners,lu2024mathvista} or finetuned LLMs on mathematics-related datasets \cite{an2024learningmistakesmakesllm,liang-etal-2024-mint,raiyan-etal-2023-math,mishra-etal-2022-lila,yue2024mammoth}. Other approaches include inference-based searching methods \cite{kang2024mindstarenhancingmathreasoning}, the application of external tools \cite{yamauchi2023lpmlllmpromptingmarkuplanguage,heyueya2023solvingmathwordproblems,chen2023program}, and the introduction of simulated interaction processes \cite{wu2024mathchatconversetacklechallenging} or self-verification mechanisms \cite{wang2023selfconsistency,zhou2024solving}.

\paragraph{Arithmetical Reasoning.} 
\citet{bubeck2023sparks} highlighted arithmetical reasoning as a key component of true mathematical ability.
However, \citet{saxton2018analysing,NEURIPS2023_deb3c281} identified significant challenges that LLMs encounter when solving elementary arithmetic tasks, such as multi-digit addition and multiplication.
A common approach to mitigate these difficulties is to reverse the output digit order \cite{shen2024positional}, or both the input and output digit order simultaneously \cite{lee2024teaching}. 
Other studies have focused on developing improved positional encodings \cite{golkar2024xval,mcleish2024transformersarithmeticrightembeddings} or positional tokens \cite{nogueira2021investigatinglimitationstransformerssimple} that are more suitable for arithmetic tasks. 
\citet{zhou2024what,zhou2024transformers} further examined the length extrapolation capabilities of LLMs in solving basic arithmetic problems, emphasizing the importance of data formats and positional embeddings for better generalization.

\subsection{Computational Powers of Transformers}

Another more relevant line of work investigates the theoretical expressive power of Transformers from a computational perspective.

\paragraph{Universal Approximation.}
Early theoretical work on Transformers primarily focused on their function approximation capabilities. \citet{yun2019transformers} demonstrated that Transformers can universally approximate any continuous sequence-to-sequence functions, given sufficient size. This universality result has since been extended to various Transformer variants, such as Sparse Transformers \cite{yun2020n}, Linear Transformers \cite{alberti2023sumformer}, and Transformers with relative positional encodings (RPE) \cite{luo2022your}. 
Additionally, previous studies established that infinite-precision Transformers are Turing-complete \cite{perez2019turing,perez2021attention}, while \citet{wei2022statistically} showed that finite-precision Transformers are approximately Turing-complete. Although these results highlight Transformers' computational capacity, our work develops expressiveness results under more practical settings, exploring the differences in expressiveness across varying levels of numerical precision.

\paragraph{Formal Language Learning.}
Another line of research focuses on the ability of Transformers to learn formal languages. 
\citet{liu2023transformers} explored how Transformers simulate finite state automata, while \citet{bhattamishra2020ability,yao2021self} studied their ability to recognize counter languages and Dyck languages, respectively. 
On the negative side, \citet{hahn2020theoretical} showed that Transformers are not capable of learning distributions over languages. 
In addition to affirmative results, several works have characterized the limitations of Transformers from the perspective of formal language modeling \citep{hahn2020theoretical,bhattamishra2020ability,weiss2021thinking,yao2021self,david2023tighter} or circuit simulation \citep{hao2022formal,merrill2022saturated,merrill2023parallelism}. 
However, few of these studies focus on the autoregressive Transformers commonly used in LLMs, which we investigate in this paper.

\paragraph{Chain-of-Thought and In-Context Learning.}
Chain-of-Thought prompting \citep{wei2022chain} plays a crucial role in tasks requiring complex reasoning structures, and several studies aim to understand its underlying mechanisms. For instance, \citet{feng2023towards,li2024chain} analyzed CoT from an expressiveness perspective, and \citet{yang2024do,wen2024rnnstransformersyetkey} examined CoT across more different model variants. In-context learning \cite{brown2020language,garg2022what} is another powerful aspect of LLMs. Some theoretical work has shown that in-context learning can be explained through gradient descent \cite{akyurek2022learning,dai2023can,von2023transformers}, while others attribute it to the induction heads mechanism \cite{elhage2021mathematical,olsson2022context}.

\subsection{Scaling Laws of Precision}
Concurrent works \cite{kumar2024scalinglawsprecision,ouyang2024lowbitquantizationfavorsundertrained} explore the impact of numerical precision on scaling laws, particularly in the contexts of training and quantization. \citet{kumar2024scalinglawsprecision} introduced ``precision-aware'' scaling laws, demonstrating that low-precision training effectively reduces a model’s ``effective parameter count''  but may still be compute-optimal for larger models. Their framework unifies the effects of both training and post-training quantization. \citet{ouyang2024lowbitquantizationfavorsundertrained} examined quantization-induced degradation (QiD), showing that larger or undertrained models exhibit greater robustness to low-bit quantization, whereas fully trained models experience significant performance degradation. While both studies underscore precision as a critical dimension in scaling laws, they leave theoretical gaps in understanding the role of precision for LLMs. Our work focuses on addressing these gaps by analyzing the impact of numerical precision on elementary arithmetic reasoning tasks.

\section{Additional Background and Preliminary}
\label{app:bg}

\subsection{Circuit Complexity}

Circuit complexity classes capture various aspects of computational complexity, typically bounding circuit width and depth. For a more detailed introduction, we refer to \citet{arora2009computational}.

We begin by defining Boolean circuits. A Boolean circuit over a basis of gates is represented as a finite-size directed acyclic graph (DAG), where each vertex corresponds to either a basis function (or gate) or an input bit. Some internal nodes are designated as outputs, and the \textit{fan-in} of a vertex is defined as its in-degree.
Building on this definition, we can define the complexity classes $\mathsf{NC}^i$, $\mathsf{AC}^i$, and $\mathsf{TC}^i$:

\begin{itemize} \item $\mathsf{NC}^i$: This class consists of constant fan-in, polynomial-sized circuits made up of AND, OR, and NOT gates, with a depth of $O(\log^i n)$. \item $\mathsf{AC}^i$: This class includes unbounded fan-in, polynomial-sized circuits composed of AND, OR, and NOT gates (with NOT gates allowed only on inputs), also having a depth of $O(\log^i n)$. \item $\mathsf{TC}^i$: This class extends $\mathsf{AC}^i$ by allowing majority gates. \end{itemize}

The relationships among the $\mathsf{NC}$, $\mathsf{AC}$, and $\mathsf{TC}$ hierarchies are as follows:

\begin{equation*} \mathsf{NC}^i \subset \mathsf{AC}^i \subset \mathsf{TC}^i \subset \mathsf{NC}^{i+1}, \; \mathsf{NC}^0 \subsetneq \mathsf{AC}^0 \subsetneq \mathsf{TC}^0. \end{equation*}

\subsection{Constant-precision Transformer}
\label{sec:def_constant_precision}

Previous work has investigated the expressiveness of constant-precision Transformers \cite{li2024chain}, utilizing a simplified version of the IEEE 754 standards \cite{8766229}. Our constant-precision setting is analogous, and we will introduce the floating-point representations we consider here.

\begin{definition}
A $(e + 2s + 1)$-floating point representation includes $e$ exponent bits, $2s$ precision bits, and one sign bit. The numbers representable under this representation are defined as follows:
\begin{equation*} 
\sF_{e,s} := \{S \cdot 2^{-s+E} \mid -2^{-2s} + 1 \leq S \leq 2^{2s} - 1, -2^{e-1} \leq E \leq \max(2^{e-1} - 1, 0), S, E \in \sZ\}. 
\end{equation*}
For any $x \in \sR$, its representation under this floating-point format is determined by rounding to the nearest value in $\sF$. In the event of a tie, we select the number with the smaller absolute value. 
\end{definition}

In this paper, we focus on the case where $e = 0$, which means all representable numbers take the form $S \cdot 2^{-s}$, with $S \in \sZ$ such that $-2^{-2s} + 1 \leq S \leq 2^{2s} - 1$. 
However, this is necessary only for \Cref{thm:con-int-add-lower}, while \Cref{thm:con-iter-add-con_dep,thm:con-mul-mod-con_dep} do not depend on specific numerical representations. 

\begin{definition}[Constant-Precision Transformer]
    A \emph{constant-precision Transformer} is a Transformer in which each neuron and activation are restricted to using a constant number of bits for computation.
\end{definition}

\citet{li2024chain} demonstrated that constant-precision Transformers with constant depth belong to the complexity class $\mathsf{AC}^0$.

\subsection{Logarithmic-precision Transformer} 
\label{sec:def_log_precision}
A key limitation of constant-precision representation is that it fails to capture the input size $n$ within a single neuron. To address this, we consider logarithmic precision, allowing for $O(\log n)$ bits for numerical representations. 

\begin{definition}[Logarithmic-Precision Transformer]
    A \emph{logarithmic-precision Transformer} is a Transformer in which each neuron and activation are allowed to use $O(\log n)$ bits for computation, where $n$ denotes the size of the input.
\end{definition}

Logarithmic-precision Transformers possess several advantageous properties \cite{feng2023towards,feng2023rethinking}:

\begin{itemize}
    \item For floating-point representations with $O(\log n)$ bits, any real number $x \in O(\mathrm{poly}(n))$  can be represented with $ O(\mathrm{poly}(1/n)) $ error.
    \item Each neuron in the Transformer can only store $O(\log n)$ bits of information, which means it cannot retain all input data. Consequently, computation must be distributed across the network, aligning with the operational principles of Transformers.
\end{itemize}

Previous work \citep{merrill2022saturated,merrill2023parallelism} has shown that logarithmic-precision Transformers fall within the complexity class \( \mathsf{TC}^0 \).

\subsection{Tokenization Scheme}
\label{sec:tokenization}

In this section, we formalize the tokenization scheme adopted in this paper and provide the necessary definitions and examples to establish a foundation for the subsequent analysis.

\begin{definition}[Tokenizer $\mT_c$]
    Let $\mathbf{x} = (x_{n-1} \cdots x_0)_p$ denote an $n$-digit integer in base $p$. The tokenizer $\mT_c$ maps $\mathbf{x}$ into $k = \lceil \frac{n}{c} \rceil$ tokens, represented as $\bm{t} = [t_{k-1}, \cdots, t_0]$, where 
    \begin{equation*}
    \begin{aligned}
        t_i = 
        \begin{cases}
            [x_{ic}, x_{ic+1}, \cdots, x_{ic+c-1}], & i < k; \\
            [x_{ic}, x_{ic+1}, \cdots, x_{n-1}], & i = k.
        \end{cases}
    \end{aligned}
    \end{equation*}
    Furthermore, for any operator (e.g., ``$+$'', ``$\times$'', ``$=$''), the tokenizer $\mT_c$ assigns each operator a single token.
\end{definition}

\begin{example}
    Consider the $5$-digit integer $13215$. Under the tokenizer $\mT_3$, it is tokenized into $[13, 215]$.
\end{example}

A key property of this tokenization scheme is that, for any fixed base $p$ and tokenizer $\mT_c$, the resulting tokenized sequence can be reinterpreted as an arithmetic expression in base $p^c$. Specifically, there exists a one-to-one mapping $\tau$ between the vocabulary of the base-$p$ tokenizer $\mT_c$ and the vocabulary of the base-$p^c$ tokenizer $\mT_1$, such that
\begin{equation*}
    \tau(\mT_c([a_{c-1}, \cdots, a_0]_p)) = \mT_1\left(\left[\sum_{i \in [c]} a_i p^i\right]\right).
\end{equation*}

\begin{proposition}
    Let $\va$ be an integer. If $\vt = \mT_c(\va) = [t_{k-1}, \cdots, t_0]$ and $\vt^\prime = \mT_1(\va) = [t_{k-1}^\prime, \cdots, t_0^\prime]$, then for all $i$, we have $\tau(t_i) = t_i^\prime$.
    \label{prop:tokenizer}
\end{proposition}

This property is particularly significant as it allows us to abstract away the specific effects of the tokenizer $\mT_c$ and focus exclusively on the case where the tokenizer is $\mT_1$. This simplification is leveraged in proving the main theorems presented in this paper.

\Cref{eg:tokenization-base} illustrates how a tokenized sequence in base-10, generated using the tokenizer $\mT_3$, can be equivalently interpreted as a sequence in base-1000.

\begin{example}
    Consider the base-10 arithmetic expression $44505 + 9416 = 53921$. When tokenized using $\mT_3$, the sequence becomes $[44, 505, +, 9, 416, =, 53, 921]$. This tokenized representation can then be reinterpreted as an arithmetic expression in base-$1000$.
    \label{eg:tokenization-base}
\end{example}

\section{Technical Lemmas}
\label{app:lem}

\subsection{Technical Lemmas for Logarithmic Precision MLP}

In this subsection, we present several foundational results concerning logarithmic precision multi-layer perceptrons (MLPs), as introduced in \cite{feng2023towards}. For brevity, proofs of these results are omitted here but are available in the appendix of \cite{feng2023towards}.

\begin{lemma}[\citealp{feng2023towards}, Lemma C.1]
\label{lemma:MLP_multip}
    Let $\epsilon > 0$. There exists a two-layer MLP $f:\mathbb{R}^2 \to \mathbb{R}$ with four hidden units and $\mathrm{GeLU}$ activation, such that for any $a, b \in [-M, M]$, the inequality $|f(a, b) - ab| \leq \epsilon$ holds. Furthermore, the $\ell_\infty$ norm of $f$ is bounded by $O(\mathrm{poly}(M, 1/\epsilon))$.
\end{lemma}

\begin{lemma}[\citealp{feng2023towards}, Lemma C.2]
\label{lemma:MLP_relu}
    Let $\vg:\mathbb{R}^{d_1} \to \mathbb{R}^{d_2}$ be a two-layer MLP with $\mathrm{ReLU}$ activation and $\ell_\infty$ norm bounded by $M$. Then, for any $\epsilon > 0$, there exists a two-layer MLP $\vf$ of the same size with $\mathrm{GeLU}$ activation such that for all $\vx \in \mathbb{R}^{d_1}$, the inequality $\|\vf(\vx) - \vg(\vx)\|_\infty \leq \epsilon$ is satisfied. Moreover, the $\ell_\infty$ norm of $\vf$ is bounded by $O(\mathrm{poly}(M, 1/\epsilon))$.
\end{lemma}

\begin{lemma}[\citealp{feng2023towards}, Lemma C.4]
\label{lemma:MLP_select}
    Consider the selection function $\vg:\mathbb{R}^d \times \mathbb{R}^d \times \mathbb{R} \to \mathbb{R}^d$ defined as 
    \[
    \vg(\vx, \vy, t) = 
    \begin{cases} 
    \vx & \text{if } t > 0, \\ 
    \vy & \text{otherwise}.
    \end{cases}
    \]
    For any $\epsilon > 0$, $\alpha > 0$, and $M > 0$, there exists a two-layer MLP $\vf$ with $2d + 2$ hidden units and $\mathrm{GeLU}$ activation such that, for all $\vx \in [-M, M]^d$, $\vy \in [-M, M]^d$, and $t \in (-\infty, -\alpha] \cup [\alpha, +\infty)$, the inequality $\|\vf(\vx, \vy, t) - \vg(\vx, \vy, t)\|_\infty \leq \epsilon$ holds. Furthermore, the $\ell_\infty$ norm of $\vf$ is bounded by $O(\mathrm{poly}(M, 1/\alpha, 1/\epsilon))$.
\end{lemma}

\subsection{Technical Lemmas for Logarithmic Precision Attention Layer}
\label{sec:lem-attn}

\citet{feng2023towards} investigated the expressive power of the standard attention layer and introduced two fundamental operations: \textbf{COPY} and \textbf{MEAN}, demonstrating that a standard attention layer with logarithmic precision can perform these operations under certain regularity conditions. In this subsection, we restate their results and extend the discussion to a specialized operation referred to as \textbf{SINGLE COPY}.

Consider a sequence of vectors $\vx_1, \vx_2, \dots, \vx_n$, where each $\vx_i = (\tilde{\vx}_i, r_i, 1) \in [-M, M]^{d+2}$, and $M$ is a fixed constant. Let the attention matrices be $\mK, \mQ, \mV \in \mathbb{R}^{d' \times (d+2)}$, and define the following transformed vectors:
\[
\vq_i = \mQ \vx_i, \quad \vk_j = \mK \vx_j, \quad \vv_j = \mV \vx_j.
\]
For any scalars $0 < \rho, \delta < M$, define the \textit{matching set} as:
\[
\gS_i = \{j \leq i : |\vq_i \cdot \vk_j| \leq \rho\}.
\]
Using this matching set, we define the following operations:
\begin{itemize}
\item \textbf{COPY}: The output is a sequence of vectors $\vu_1, \dots, \vu_n$, where
  \[
  \vu_i = \vv_{\mathrm{pos}(i)}, \quad \text{with } \mathrm{pos}(i) = \operatorname{argmax}_{j \in \gS_i} r_j.
  \]
  The output $\vu_i$ is undefined if $\gS_i = \varnothing$.

\item  \textbf{MEAN}: The output is a sequence of vectors $\vu_1, \dots, \vu_n$, where
  \[
  \vu_i = \operatorname{mean}_{j \in \gS_i} \vv_j = \frac{1}{|\gS_i|} \sum_{j \in \gS_i} \vv_j.
  \]
  The output $\vu_i$ is undefined if $\gS_i = \varnothing$.

\item  \textbf{SINGLE COPY}: The output is a sequence of vectors $\vu_1, \dots, \vu_n$, where
  \[
  \vu_i = \vv_{\mathrm{pos}(i)}, \quad \text{with } \mathrm{pos}(i) \text{ being the unique element in } \gS_i.
  \]
  The output $\vu_i$ is undefined if $|\gS_i| \neq 1$.
\end{itemize}

We now impose the following regularity assumption to ensure the feasibility of the operations under consideration:

\begin{assumption}[Regularity Assumption for Attention]
\label{ass:attention}
    For any input sequence $\vx_1, \vx_2, \dots, \vx_n$, the matrices $\mQ, \mK, \mV$ and scalars $\rho, \delta$ satisfy the following conditions:
    \begin{itemize}
        \item For any $i, j \in [n]$, either $|\vq_i \cdot \vk_j| \leq \rho$ or $\vq_i \cdot \vk_j \leq -\delta$.
        \item For any $i, j \in [n]$, either $i = j$ or $|r_i - r_j| \geq \delta$.
        \item The infinity norm of the value matrix $\mV$ satisfies $\|\mV\|_\infty \leq 1$.
    \end{itemize}
\end{assumption}

Under this assumption, we demonstrate that a logarithmic precision attention layer with $O(d)$ embedding dimension and a single attention head can perform the operations defined in Section~\ref{sec:lem-attn}.

\begin{lemma}[\citealp{feng2023towards}, Lemma C.7]
\label{lem:attn-log-copy}
    Suppose Assumption~\ref{ass:attention} holds and $\rho \leq \frac{\delta^2}{8M}$. For any $\epsilon > 0$, there exists an attention layer with a single attention head and $O(d)$ embedding dimension that can approximate the $\mathrm{COPY}$ operation. Furthermore, the $\ell_\infty$ norm of the parameters is bounded by $O(\mathrm{poly}(M, 1/\delta, \log(n), \log(1/\epsilon)))$. 

    Formally, for any input sequence $\vx_1, \vx_2, \dots, \vx_n$, let the attention layer outputs be $\vo_1, \vo_2, \dots, \vo_n$. Then, for any $i \in [n]$ such that $\gS_i \neq \varnothing$, the following holds:
    \[
    \|\vo_i - \vu_i\|_\infty \leq \epsilon,
    \]
    where $\vu_i$ is the target output of the $\mathrm{COPY}$ operation as defined in Section~\ref{sec:lem-attn}.
\end{lemma}

\begin{lemma}[\citealp{feng2023towards}, Lemma C.8]
\label{lem:attn-log-mean}
    Suppose Assumption~\ref{ass:attention} holds and $\rho \leq \frac{\delta \epsilon}{16M \ln(4Mn/\epsilon)}$. For any $0 < \epsilon \leq M$, there exists an attention layer with a single attention head and $O(d)$ embedding dimension that can approximate the $\mathrm{MEAN}$ operation. Furthermore, the $\ell_\infty$ norm of the parameters is bounded by $O(\mathrm{poly}(M, 1/\delta, \log(n), \log(1/\epsilon)))$. 

    Formally, for any input sequence $\vx_1, \vx_2, \dots, \vx_n$, let the attention layer outputs be $\vo_1, \vo_2, \dots, \vo_n$. Then, for any $i \in [n]$ such that $\gS_i \neq \varnothing$, the following holds:
    \[
    \|\vo_i - \vu_i\|_\infty \leq \epsilon,
    \]
    where $\vu_i$ is the target output of the $\mathrm{MEAN}$ operation as defined in Section~\ref{sec:lem-attn}.
\end{lemma}

The proofs of Lemmas~\ref{lem:attn-log-copy} and~\ref{lem:attn-log-mean} are omitted here for brevity. Complete proofs can be found in the appendix of \citet{feng2023towards}.

\begin{lemma}
\label{lem:attn-log-single-copy}
    Suppose Assumption~\ref{ass:attention} holds and $\delta - \rho \geq c \rho$ for some constant $c > 0$. For any $\epsilon > 0$, there exists an attention layer with a single attention head and $O(d)$ embedding dimension that can approximate the $\mathrm{SINGLE~COPY}$ operation. Furthermore, the $\ell_\infty$ norm of the parameters is bounded by $O(\mathrm{poly}(M, 1/\delta, 1/c, \log(n), \log(1/\epsilon)))$. 

    Formally, for any input sequence $\vx_1, \vx_2, \dots, \vx_n$, let the attention layer outputs be $\vo_1, \vo_2, \dots, \vo_n$. Then, for any $i \in [n]$ such that $|\gS_i| = 1$, the following holds:
    \[
    \|\vo_i - \vu_i\|_\infty \leq \epsilon,
    \]
    where $\vu_i$ is the target output of the $\mathrm{SINGLE~COPY}$ operation, as defined in Section~\ref{sec:lem-attn}.
\end{lemma}

\begin{proof}
    We construct the query, key, and value vectors as follows:
    \begin{itemize}
        \item Query: $\lambda \vq_i \in \mathbb{R}^d$
        \item Key: $\vk_i \in \mathbb{R}^d$
        \item Value: $\vv_i \in \mathbb{R}^d$
    \end{itemize}
    where $\lambda > 0$ is a constant to be determined. Denote $a_{i,j}$ as the attention score, defined as:
    \[
    a_{i,j} = \frac{\exp(\lambda (\vq_i \cdot \vk_j))}{\sum_{j'} \exp(\lambda (\vq_i \cdot \vk_{j'}))}.
    \]

    Since $\delta - \rho \geq c \rho$, it follows that $\delta - \rho \geq \frac{c}{c + 1} \delta$. Setting 
    \[
    \lambda = \frac{(c + 1) \ln\left(\frac{2nM}{\epsilon}\right)}{c \delta},
    \]
    which is bounded by $O(\mathrm{poly}(M, 1/\delta, 1/c, \log(n), \log(1/\epsilon)))$, we derive the following bounds for $a_{i, \mathrm{pos}(i)}$:
    \begin{align}
    \label{eqn:proof-single-1}
        a_{i, \mathrm{pos}(i)} 
        &\geq \frac{\exp(-\lambda \rho)}{\exp(-\lambda \rho) + (n-1) \exp(-\lambda \delta)} \\
    \notag
        &= \frac{1}{1 + (n-1) \exp(-\lambda (\delta - \rho))} \\
    \label{eqn:proof-single-2}
        &\geq 1 - (n-1) \exp(-\lambda (\delta - \rho)) \\
        &\geq 1 - n \exp\left(- \ln\left(\frac{2nM}{\epsilon}\right)\right) \notag \\
        &= 1 - \frac{\epsilon}{2M}. \notag
    \end{align}
    Here,
    \Cref{eqn:proof-single-1} follows from \Cref{ass:attention} and the condition $|\gS_i| = 1$, which ensures that for $j' \neq \mathrm{pos}(i)$, $\vq_i \cdot \vk_{j'} \leq -\delta$;
    \Cref{eqn:proof-single-2} uses the approximation $\frac{1}{1 + x} \geq 1 - x$ for $x \geq 0$.

    Thus, we can bound the error as follows:
    \begin{align*}
        \|\vo_i - \vu_i\|_\infty 
        &= \left\|\sum_{j} a_{ij} \vv_j - \vv_{\mathrm{pos}(i)}\right\|_\infty \\
        &\leq M \|\mV\|_\infty \cdot \left(1 - a_{i, \mathrm{pos}(i)} + \sum_{j \neq \mathrm{pos}(i)} a_{ij}\right) \\
        &= M \|\mV\|_\infty \cdot (2 - 2a_{i, \mathrm{pos}(i)}) \\
        &\leq \epsilon,
    \end{align*}
    where the last inequality follows from the bound $a_{i, \mathrm{pos}(i)} \geq 1 - \frac{\epsilon}{2M}$ and the constraint $\|\mV\|_\infty \leq 1$ from Assumption~\ref{ass:attention}. 
    This concludes the proof.
\end{proof}

\subsection{Technical Lemmas for Constant Precision Calculations}

In this section, we establish technical lemmas that underpin constant precision calculations. Assume a system with $2s$-bit fixed-point precision and no exponent bits, and let $B_s = 2^s - 2^{-s}$. The largest representable value in this system is $B_s$, while the smallest is $-B_s$.

\begin{lemma}[\citealp{li2024chain}, Lemmas E.1 and E.2]
\label{lem:con-pre-overflow}
    For any $s \in \mathbb{N}_+$, it holds that $\exp(-B_s) = 0$ and $\exp(B_s) = B_s$.
\end{lemma}

\begin{proof}
    First, observe that $\exp(B_s) \geq \mathrm{e} B_s > 2^{s+1}$. Consequently, $\exp(-B_s) \leq 2^{-s-1}$, implying $\exp(-B_s) = 0$ due to the truncation to zero under the given precision. For the second claim, note that $\exp(B_s) \geq B_s + 1 > B_s$, which enforces $\exp(B_s) = B_s$ under the constant precision constraints.
\end{proof}

\begin{lemma}
\label{lem:con-pre-gelu}
    For any $s \in \mathbb{N}_+$, we have $\mathrm{GeLU}(-B_s) = 0$.
\end{lemma}

\begin{proof}
    To prove this, it suffices to show that $B_s \Phi(-B_s) \leq 2^{-s-1}$, where $\Phi$ denotes the cumulative distribution function (CDF) of the standard Gaussian distribution.

    \textbf{Case 1:} $s = 1$. In this case, $B_s = \frac{3}{2}$. Thus,
    \begin{equation*}
        B_s \Phi(-B_s) \leq \frac{3}{2} \Phi(-1) \leq \frac{3}{2} \cdot \frac{1 - 0.68}{2} < \frac{1}{4}.
    \end{equation*}

    \textbf{Case 2:} $s \geq 2$. For larger $s$, we proceed as follows:
    \begin{equation*}
    \begin{aligned}
        B_s \Phi(-B_s) 
        &= \frac{B_s}{\sqrt{2\pi}} \int_{B_s}^{+\infty} \rme^{-\frac{x^2}{2}} \, \rmd x \\
        &\leq \frac{B_s}{\sqrt{2\pi}} \int_{B_s}^{+\infty} \rme^{-\frac{B_s x}{2}} \, \rmd x \\
        &\leq \sqrt{\frac{2}{\pi}} \rme^{-\frac{B_s^2}{2}} 
         \leq \frac{2\sqrt{2}}{\sqrt{\pi}(B_s^2 + 2)}  \\
       &  \leq \frac{2\sqrt{2}}{\sqrt{\pi} \cdot 2^{2s}} \leq \frac{1}{2^{s+1}}.
    \end{aligned}
    \end{equation*}

    This completes the proof.
\end{proof}

\newenvironment{thmbis}[1]
  {\renewcommand{\thetheorem}{\ref{#1}}%
   \addtocounter{theorem}{-1}%
   \begin{theorem}}
  {\end{theorem}}

\section{Proofs for Section \ref{sec:constant-precision}}
\label{app:proof-constant}

In this section, we present the formal proofs of the theorems stated in \Cref{sec:constant-precision}. Before delving into the proofs, we revisit the role of the tokenizer $\mT_c$. As established in \Cref{sec:tokenization} and \Cref{prop:tokenizer}, it suffices to focus on the case of $\mT_1$, where both the inputs and outputs are tokenized into single digits. This simplification is key to the subsequent analysis and constructions.

\subsection{Proof for Theorem \ref{thm:con-int-add-lower}}
\label{sec:proof:con-int-add-lower}

\begin{thmbis}{thm:con-int-add-lower}
Fix integers $p \geq 2$ and $c \in \mathbb{N}^*$. Consider the tokenizer $\mT_c$ defined in \cref{eq:tokenizer} for processing the input and output sequences. There exist constant-precision Transformers with constant depth $L$ (independent of $n$) and hidden dimension $d = O(n^2)$ that can solve the $\operatorname{ADD}_p(n)$ task.
\end{thmbis}

To aid readability, we first describe an algorithm to perform the $\operatorname{ADD}_p(n)$ task (\Cref{alg:ADD}) and prove its correctness. Subsequently, we construct a Transformer with the specified configuration in \Cref{thm:con-int-add-lower} that simulates \Cref{alg:ADD}.

\begin{algorithm}
    \SetKwInOut{Input}{Input}
    \SetKwInOut{Output}{Output}
    \caption{$\operatorname{ADD}_p(n)$ Algorithm}
    \label{alg:ADD}
    \SetAlFnt{\small\sf} % Set font for algorithm environment
    \Input{Two $p$-adic numbers $\va, \vb$ with lengths $n_1$ and $n_2$, respectively.}
    \Output{The sum of the inputs, $\vo$, represented as a $p$-adic number with $(n+1)$ digits, where $n = \max(n_1, n_2)$.}
    \BlankLine
    Initialize $a_n = 0$ and $b_n = 0$\;
    \ForEach{$i \in \{0, \cdots, n-1\}$}{
        Compute the carry-on bits $\vc$\;
        $i_\wedge = \max \{j \leq i \mid a_j + b_j \geq p\}$\;
        $i_\vee = \max \{j \leq i \mid a_j + b_j \leq p - 2\}$\;
        $c_i = \vone_{i_\wedge > i_\vee}$\;
    }
    Compute the output digits $\vo$: $o_i = (a_i + b_i + c_{i-1}) \mod p$\;
\end{algorithm}

\begin{lemma}[An algorithm to perform $\operatorname{ADD}_p(n)$]
\label{lmm:alg:add}
\Cref{alg:ADD} outputs $\vo = \va + \vb$ for all inputs $\va, \vb$.
\end{lemma}

\begin{proof}
    Consider two $n$-bit $p$-adic numbers $\va$ and $\vb$. The carry-over bits $\vc = (c_n, \cdots, c_1)$ can be computed recursively as follows:  
    \begin{equation}
        \begin{aligned}
            & c_{-1} = 0, \\
            & c_0 = \vone_{a_0 + b_0 \geq p}, \\
            & c_1 = (c_0 \cdot \vone_{a_1 + b_1 \geq p - 1}) \vee \vone_{a_1 + b_1 \geq p}, \\
            & \cdots, \\
            & c_i = (c_{i-1} \cdot \vone_{a_i + b_i \geq p - 1}) \vee \vone_{a_i + b_i \geq p}.
        \end{aligned}
    \end{equation}
    To avoid the recursive computation, the carry-over bits can be expressed in closed form as:
    \begin{equation}
    \label{eq:add:carry:on}
        \begin{aligned}
            & i_\wedge = \max \{j \leq i \mid a_j + b_j \geq p \}, \\
            & i_\vee = \max \{j \leq i \mid a_j + b_j \leq p - 2\}, \\
            & c_i = \vone_{i_\wedge > i_\vee}.
        \end{aligned}
    \end{equation}
    Alternatively, the carry-over bits can be expressed equivalently as:
    \begin{equation}
        c_i = \bigvee_{0 \leq j \leq i} \left[ \vone_{a_j + b_j \geq p} 
        \wedge \bigwedge_{j \leq k \leq i} \vone_{a_k + b_k \geq p - 2} \right].
        \label{eqn:add-result-formula}
    \end{equation}

    In \Cref{eq:add:carry:on}, $i_\wedge$ identifies the largest bit index less than or equal to $i$ that contributes a carry to higher bits, while $i_\vee$ identifies the largest bit index less than or equal to $i$ such that the carry generated below $i_\vee$ does not propagate beyond $i_\vee$. Thus, the carry-over bit $c_i = 1$ if and only if $i_\wedge > i_\vee$.

    After computing the carry-over bits, the sum of the input integers can be computed as:
    \begin{equation}
        \begin{aligned}
            & o_0 = (a_0 + b_0) \mod p, \\
            & o_1 = (a_1 + b_1 + c_0) \mod p, \\
            & \cdots \\
            & o_i = (a_i + b_i + c_{i-1}) \mod p, \\
            & o_n = c_{n-1}.
        \end{aligned}
    \end{equation}

    Therefore, the output $\vo$ is exactly the sum of the two input numbers, and \Cref{alg:ADD} correctly computes $\operatorname{ADD}_p(\va, \vb)$ for all $\va, \vb \in \{0, 1\}^n$.
\end{proof}

Next, we provide the proof for \Cref{thm:con-int-add-lower}.

\begin{proof}[Proof for \Cref{thm:con-int-add-lower}]
    We now demonstrate that a constant-precision Transformer, with constant depth $L$, a fixed number of attention heads, and a hidden dimension of size $O(n^2)$, is capable of simulating \Cref{alg:ADD}. Consequently, this model can accurately produce the correct output for any pair of input integers $\va$ and $\vb$.

    \textbf{Initial Embeddings:} The total length of the input sequence is at most $2(n + 1)$. We categorize the tokens into two distinct classes: 
    (1) \emph{number tokens} representing digits ($0, 1, \cdots, p - 1$), and 
    (2) \emph{auxiliary tokens} for operations and control flow (``$+$'', ``$=$'', \texttt{<SOS>}, and \texttt{<EOS>}).

    The embeddings for each token are initialized as follows:
    \begin{itemize}
        \item \textbf{Embedding of input token $a_i$:} 
        \[
        \vu_{a,i}^0 = (a_i \ve_{i+1}, \vzero, -1, \vzero, 0, 1, 1).
        \]
        \item \textbf{Embedding of input token $b_i$:} 
        \[
        \vu_{b,i}^0 = (\vzero, b_i \ve_{i+1}, -1, \vzero, 0, 2, 1).
        \]
        \item \textbf{Embedding of output token $o_i$:} 
        \[
        \vu_{o,i}^0 = (\vzero, \vzero, o_i, \ve_{i+1}, 0, 3, -1).
        \]
        \item \textbf{Embedding of the ``$+$'' token:} 
        \[
        \vu_{+}^0 = (\vzero, \vzero, -1, \vzero, 0, 4, -1).
        \]
        \item \textbf{Embedding of the ``$=$'' token:} 
        \[
        \vu_{=}^0 = (\vzero, \vzero, -1, \vzero, 1, 5, -1).
        \]
        \item \textbf{Embedding of the \texttt{<SOS>} token:} 
        \[
        \vu_{\texttt{<SOS>}}^0 = (\vzero, \vzero, -1, \vzero, 0, 6, -1).
        \]
        \item \textbf{Embedding of the \texttt{<EOS>} token:} 
        \[
        \vu_{\texttt{<EOS>}}^0 = (\vzero, \vzero, 0, \vzero, 0, 3, -1).
        \]
    \end{itemize}
    In each of these embeddings:
    \begin{itemize}
        \item $\ve_i \in \mathbb{R}^{n+1}$ is a one-hot vector representing the positional encoding of the token (e.g., digit $a_i$) in the sequence.
        \item $\vzero$ is a vector of zeros of appropriate dimensions.
    \end{itemize}

    \textbf{Block 1.} The first block of the Transformer performs the \textbf{COPY} operation, which copies the values of $a_i, b_i$ to the positions of the output tokens. This is achieved using the attention mechanism. The query, key, and value are set as follows:

\begin{itemize}
    \item \textbf{Query:} $\vq = B_s$
    \item \textbf{Key:} $\vk = \vu^0[3n + 6]$, i.e., $\vk = 1$ for input number tokens, and $\vk = -1$ otherwise.
    \item \textbf{Value:} $\vv = \vu^0[1, \cdots, 2n+2]$, i.e., 
    \[
    \vv = 
    \begin{cases} 
        (a_i \ve_{i+1}, \vzero) & \text{for input $\va$,} \\
        (\vzero, b_i \ve_{i+1}) & \text{for input $\vb$,} \\
        \vzero & \text{otherwise.}
    \end{cases}
    \]
\end{itemize}

Since we operate under constant precision, we carefully analyze the attention values. The attention value (before normalization) is $B_s$ for tokens $a_i, b_i$ and $-B_s$ otherwise. By \Cref{lem:con-pre-overflow}, we know $\exp(B_s) = B_s$ and $\exp(-B_s) = 0$. The normalization term for attention is $2n B_s = B_s$, so the attention weights are $1$ for tokens $a_i, b_i$ and $0$ otherwise. As a result, the attention output at the positions of the output tokens is always $(a_0, \cdots, a_n, b_0, \cdots, b_n)$.

\textbf{Block 2.} The second block of the Transformer uses MLPs to compute the output $\vo$ based on \Cref{alg:ADD}. The calculations proceed in the following steps:
\begin{itemize}
    \item \textbf{Compute $r_i = a_i + b_i$ for $i = 0, \cdots, n$.}

    This can be implemented using an MLP with constant hidden dimension. To avoid overflow of $r_i$, we require $B_s \geq 2p$.
    \item \textbf{Compute $f_i = \vone_{r_i \geq p}$ and $g_i = \vone_{r_i \geq p - 2}$.}

    Using \Cref{lem:con-pre-gelu}, we can calculate:
   \[
   f_i = \frac{\text{GeLU}[B_s \cdot (2r_i - 2p + 1)]}{\text{GeLU}(B_s)}, \quad 
   g_i = \frac{\text{GeLU}[B_s \cdot (2r_i - 2p + 5)]}{\text{GeLU}(B_s)}.
   \]
   Here, we require $B_s \geq 4p$ to avoid overflow of $2r_i - 2p + 1$.
   \item \textbf{Compute $c_i$ using the formula:}
   \[
        c_i = \bigvee_{0\leq j\leq i}\left[ \vone_{a_j + b_j \geq p} \wedge\bigwedge_{j \leq k \leq i} \vone_{a_k + b_k \geq p - 2} \right] =  \bigvee_{0\leq j\leq i}\left[ f_j \wedge\bigwedge_{j \leq k \leq i} g_k \right].
   \]
   Notice that:
   \[
   \bigvee_{1 \leq i \leq \gamma} \alpha_i = \frac{\text{GeLU} \left[ B_s \cdot \left( \sum_{i=1}^\gamma \alpha_i \right) \right]}{\text{GeLU}(B_s)}, \quad
   \bigwedge_{1 \leq i \leq \gamma} \alpha_i = 1 - \bigvee_{1 \leq i \leq \gamma} (1 - \alpha_i).
   \]
   These formulas imply that the $\vee$ and $\wedge$ operations can be implemented using a constant-depth, constant-precision MLP with constant hidden dimension. Therefore, $c_i$ can be computed using $O(n)$ hidden dimension.
   \item \textbf{Compute $o_i = a_i + b_i + c_{i-1}$ for $i = 0, \cdots, n$.}

      This computation can also be implemented with a constant hidden dimension. Again, we require $B_s \geq 2p$ to avoid overflow of $o_i$.
\end{itemize}

Since we need to compute $r_i, f_i, g_i, c_i$ for all $i$, the hidden dimension of this block is $O(n^2)$.

\textbf{Block 3.} This block filters out the token $o_i$ from $\vo$. Specifically, for the token $o_{i+1}$, where $i \in \{0, \cdots, n-1\}$, we predict the next token $o_i$. 

First, we calculate the positional embedding $\ve_{i+1}$ using $\ve_{i+2}$ from the positional embedding of $\vm_{o,i+1}^0$. Then, we compute $o_i$ as:
\[
o_i = \langle \ve_{i+1}, \vo \rangle.
\]
Using the property $x = \text{GeLU}(x) - \text{GeLU}(-x)$, this can be expanded as:
\[
o_i = \sum_{j=1}^{n+1} \ve_{i+1}[j] \vo[j] = 
\sum_{j=1}^{n+1} \left[ \text{GeLU}(\ve_{i+1}[j] - B_s (2 - 2 \vo[j])) - \text{GeLU}(-\ve_{i+1}[j] - B_s (2 - 2 \vo[j])) \right].
\]
Thus, $o_i$ can be calculated using $O(n)$ hidden dimension. The final output from this layer is given by:
\[
\ve_{o,i+1}^3 = 
\begin{cases} 
    (o_i, \ve_{i+1}) & \text{if \(i > 0\)}, \\
    (0, \vzero) & \text{if \(i = 0\)}.
\end{cases}
\]

\textbf{Predict Next Token.} Given the output embeddings from the last Transformer layer, $\ve_{o,i}^3$, and the word embeddings, the Transformer predicts the next token by finding the nearest word embedding.

\textbf{Precision Requirements.} In this construction, we require $B_s \geq 4p$, which guarantees that constant precision is sufficient for all computations.
\end{proof}

\subsection{Proof for Theorem \ref{thm:con-iter-add-con_dep}}
\label{sec:proof_con-iter-add-con_dep}

\begin{thmbis}{thm:con-iter-add-con_dep}
Fix integers $p\geq 2$ and $c,L\in \mathbb{N}^*$. Consider the tokenizer $\mT_c$ defined in \cref{eq:tokenizer} for processing the input and output sequences. For any polynomial $f$, there exist problem scales $n$ and $k$ such that no constant-precision autoregressive Transformer with $L$ layers and hidden dimension $d < f(n, k)$ can correctly solve the $\operatorname{IterADD}_p(n,k)$ task.
\end{thmbis}

\begin{proof}
Assume, for the sake of contradiction, that there exist integers $p \geq 2$, $L$, and a polynomial $f$, such that for all problem scales $n$ and $k$, there exists a constant-precision autoregressive Transformer with $L$ layers and hidden dimension $d \leq f(n, k)$ that can solve the $\operatorname{IterADD}_p(n, k)$ task correctly.

We now consider the majority function $\operatorname{Maj}(b_1, \dots, b_k)$, where $b_i \in \{0, 1\}$. To establish the contradiction, we construct a reduction from $\operatorname{Maj}(b_1, \dots, b_k)$ to $\operatorname{IterADD}_p(2, k')$, where $k' = p^{\lceil \log_p k \rceil} \leq pk$. Specifically, let $a_i = b_i (p^2 - 1)$ for $i = 1, \dots, k$, and define the remaining terms as follows:
\begin{equation*}
    a_{k+1} + \dots + a_{k'} = p^{\lceil \log_p k \rceil + 1} - (p^2 - 1) \left\lceil \frac{k}{2} \right\rceil.
\end{equation*}
This construction is feasible because:
\begin{equation*}
    p^{\lceil \log_p k \rceil + 1} - (p^2 - 1)\left\lceil \frac{k}{2} \right\rceil 
    \leq (p^{\lceil \log_p k \rceil} - k)(p^2 - 1),
\end{equation*}
which holds for $p \geq 2$. Consequently, the following equivalence relationships hold:
\begin{equation*}
    \operatorname{Maj}(b_1, \dots, b_k) = 1 
    \iff \sum_{i=1}^k b_i \geq \left\lceil \frac{k}{2} \right\rceil
    \iff \sum_{i=1}^{k'} a_i \geq p^{\lceil \log_p k \rceil + 1}
    \iff o_{\lceil \log_p k \rceil + 1} > 0,
\end{equation*}
where $o_{\lceil \log_p k \rceil + 1}$ is the output token corresponding to the final layer of the Transformer.

Now, observe that a bounded-depth, fixed-precision decoder-only Transformer with polynomial hidden dimension, which generates a single token, operates within the complexity class $\mathsf{AC}^0$. However, by the reduction above, solving $\operatorname{IterADD}_p(2, k')$ implies the ability to compute $\operatorname{Maj}$. This leads to a contradiction, as $\operatorname{Maj} \notin \mathsf{AC}^0$. 

Thus, no constant-precision autoregressive Transformer with $L$ layers and hidden dimension $d \leq f(n, k)$ can solve the $\operatorname{IterADD}_p(n, k)$ task in general.
\end{proof}

\subsection{Proof for Theorem \ref{thm:con-mul-mod-con_dep}}
\label{sec:proof:con-mul-mod-con_dep}

\begin{thmbis}{thm:con-mul-mod-con_dep}
Fix integers $p\geq 2$ and $c,L\in \mathbb{N}^*$. Consider the tokenizer $\mT_c$ defined in \cref{eq:tokenizer} for processing the input and output sequences. For any polynomial $f$, there exist problem scales $n$ and $l$ such that no constant-precision autoregressive Transformer with $L$ layers and hidden dimension $d < f(n, l)$ can correctly solve the $\operatorname{MUL}_p(n, l)$ task.
\end{thmbis}

\begin{proof}
Assume, for contradiction, that there exist integers $p \geq 2$, $L$, and a polynomial $f$, such that for all problem scales $n$ and $l$, there exists a constant-precision autoregressive Transformer with $L$ layers and hidden dimension $d \leq f(n, l)$ that can correctly solve the $\operatorname{MUL}_p(n, l)$ task.

Now, consider the majority function $\operatorname{Maj}(c_1, \dots, c_k)$, where $c_i \in \{0, 1\}$. We construct a reduction from $\operatorname{Maj}(c_1, \dots, c_k)$ to $\operatorname{MUL}_p(n, l)$. Specifically, let 
\begin{equation*}
    n = \left( \ceil{\log_p k} + 1 \right) \left( p^{\ceil{\log_p k}} + \left \lfloor \frac{k}{2} \right \rfloor \right) = O(k \log k), \quad
    l = n + \ceil{\log_p k} = O(k \log k).
\end{equation*}
We extend $c_i$ by defining
\begin{equation*}
    k' = p^{\ceil{\log_p k}} + \left \lfloor \frac{k}{2} \right \rfloor, \quad
    c_{k+1} = \cdots = c_{k'} = 1,
\end{equation*}
and construct the sequences $\va$ and $\vb$ as follows:
\begin{equation*}
    \va = c_1 \underbrace{0 \cdots 0}_{\ceil{\log_p n}} c_2 \underbrace{0 \cdots 0}_{\ceil{\log_p n}} \cdots c_{k'} \underbrace{0 \cdots 0}_{\ceil{\log_p n}}, \quad
    \vb = 1 \underbrace{0 \cdots 0}_{\ceil{\log_p n}} 1 \underbrace{0 \cdots 0}_{\ceil{\log_p n}} \cdots 1 \underbrace{0 \cdots 0}_{\ceil{\log_p n}}.
\end{equation*}

Under this construction, the following equivalences hold:
\begin{equation*}
    \operatorname{Maj}(c_1, \dots, c_k) = 1 \iff 
    c_1 + \cdots + c_k \geq \left \lceil \frac{k}{2} \right \rceil \iff 
    c_1 + \cdots + c_{k'} \geq p^{\ceil{\log_p k}} \iff 
    o_{l - 1} > 0,
\end{equation*}
where $o_{l-1}$ denotes the first output token.

Now, observe that a bounded-depth, fixed-precision decoder-only Transformer with polynomial hidden dimension, which generates a single token, operates within the complexity class $\mathsf{AC}^0$. However, by the reduction above, solving $\operatorname{MUL}_p(n, l)$ implies the ability to compute $\operatorname{Maj}$. This leads to a contradiction, as $\operatorname{Maj} \notin \mathsf{AC}^0$. 

Hence, no constant-precision autoregressive Transformer with $L$ layers and hidden dimension $d \leq f(n, l)$ can correctly solve $\operatorname{MUL}_p(n, l)$ task for any problem scale $n$ and $l$.
\end{proof}

\section{Proofs for Section \ref{sec:log-precision}}
\label{app:proof-log}

In this section, we provide the formal proofs of the theorems stated in \Cref{sec:log-precision}. Before proceeding with the proofs, we revisit the role of the tokenizer $\mT_c$. As established in \Cref{sec:tokenization} and \Cref{prop:tokenizer}, it is sufficient to focus on the case of $\mT_1$, where both the input and output sequences are tokenized into single digits. This simplification is crucial for the subsequent analysis and constructions. Notably, this reasoning parallels the argument presented at the beginning of \Cref{app:proof-constant} (Proof of \Cref{sec:constant-precision}).

\subsection{Proof for Theorem \ref{thm:log-int-add-con}}
\label{sec:proof_log-int-add-con}

\begin{thmbis}{thm:log-int-add-con}
Fix integers $p\geq 2$ and $c\in \mathbb{N}^*$. Consider the tokenizer $\mT_c$ defined in \cref{eq:tokenizer} for processing the input and output sequences. There exists a logarithmic-precision Transformer with constant depth and hidden dimension (independent of $n$) that can generate the correct output for any input on the $\operatorname{ADD}_p(n)$ task.
\end{thmbis}

\begin{proof}
    The result in \cref{thm:log-int-add-con} follows as a special case of \cref{thm:log-iter-add-con}. Specifically, by setting $k = 2$ in \cref{thm:log-iter-add-con}, the proof is complete. Observe that in this case, $m = \ceil{\log_p k} = 1$, which implies that the combination of neighboring bits is unnecessary.
\end{proof}

\subsection{Proof for Theorem \ref{thm:log-iter-add-con}}
\label{sec:proof_log-iter-add-con}

\begin{thmbis}{thm:log-iter-add-con}
Fix integers $p\geq 2$ and $c\in \mathbb{N}^*$. Consider the tokenizer $\mT_c$ defined in \cref{eq:tokenizer} for processing the input and output sequences. For any integers $n$ and $k$, there exists a logarithmic-precision Transformer with constant depth and hidden dimension $d$ (independent of $n$ and $k$) that can generate the correct output for any input on the $\operatorname{IterADD}_p(n, k)$ task.

\end{thmbis}

For ease of understanding, we first present an algorithm to compute $\operatorname{IterADD}_p(n, k)$ (\Cref{alg:Iter_ADD}) and prove its correctness. Subsequently, we demonstrate the construction of a constant-size Transformer with logarithmic precision to simulate \Cref{alg:Iter_ADD}.

\begin{algorithm}
 \SetKwInOut{Input}{Input}
 \SetKwInOut{Output}{Output}
\caption{$\operatorname{IterADD}_p(n, k)$ Algorithm}\label{alg:Iter_ADD}
\SetAlFnt{\small\sf}
   \Input{ $k$ $p$-adic numbers $\va_1,\cdots,\va_k$, each of maximum length $n$}
   \Output{ The sum of the inputs ${\vo}$}
   \BlankLine
   $m = \lceil \log_p k \rceil$\;
   Compute the sum of each bit: $r_j=\sum_{i\in[k]} a_{ij}$ for $j = 0, \cdots, n - 1$\;
   Combine neighboring $m$ bits: $$s_i=\sum_{j=0}^{m-1} r_{ik+j}p^{j}$$
   for $i = 0, \cdots, \lfloor n / m \rfloor$\;
   Decompose $s_i$: $s_i = b_i p^m + q_i$, where $q_i\in[0,p^m-1]$ and $b_i,q_i\in\mathbb{N}$\;
   Initialize $c_0 = 0$\;
   \ForEach{$i = 0, \cdots, \lfloor n / m \rfloor$}{
   Compute carry bits $\vc$:\\
    $i_\wedge=\max \{j\leq i \mid q_j+b_{j-1}\geq p^m\}$\;
    $i_\vee = \max \{j\leq i \mid q_j+b_{j-1}\leq p^m - 2\}$\;
    $c_i = \vone_{i_\wedge > i_\vee}$\;
   }    
   Compute the $p^m$-adic outcome $\tilde \vo$: $\tilde{o}_i=(q_i + b_{i-1} + c_{i-1})\mod p^m$ for $i = 0, \cdots, \lfloor n / m \rfloor + 1$\;
   Covert $p^m$-adic $\tilde \vo$ to $p$-adic $\vo$: 
   $${o}_i = \left\lfloor \frac{\tilde o_j\mod p^{(l+1)}}{p^l} \right\rfloor$$ for $i=jk+l$, where $l\in \{0, \cdots, k-1\}$, $j\in \sZ$\;
\end{algorithm}

\begin{lemma}[Algorithm for $\operatorname{IterADD}_p(n, k)$]
\Cref{alg:Iter_ADD} computes $\vo = \va_1 + \cdots + \va_k$ for any inputs $\va_1, \ldots, \va_k$.
\label{lmm:alg:iter:add}
\end{lemma}

\begin{proof}
    The initial four steps of the algorithm transform $p$-adic addition into $p^m$-adic addition. This transformation allows the sum of $k$ numbers to be represented as $\sum_{i} s_i p^{im}$, where $s_i$ are intermediate coefficients.

    At this stage, $s_i \in [0, kp^m)$. To ensure the final results are accurate, we must account for carry-over effects such that the outputs $\tilde{o}_i$ remain within the range $[0, p^m-1]$. Each $s_i$ can be decomposed as $s_i = b_ip^m + q_i$, where $q_i \in [0, p^m-1]$ and $b_i < k \leq p^m$. Consequently, the overflow $b_i$ propagates only \textit{directly} to the next subsequent digit $q_{i+1}$. Notably, $q_{i+1} + b_i \leq 2(p^m - 1)$. 

    Let $\vc$ denote the vector recording carry-over effects at each position $i$. The carry-over can be computed iteratively as:
    \begin{equation}
        \begin{aligned}
            & c_{-1} = 0, \\
            & c_0 = \vone_{q_0 + b_{-1} \geq p^m} \; (b_{-1} := 0), \\
            & c_1 = \big(c_0 \cdot \vone_{q_1 + b_0 \geq p^m - 1}\big) \vee \vone_{q_1 + b_0 \geq p^m}, \\
            & \cdots \\
            & c_i = \big(c_{i-1} \cdot \vone_{q_i + b_{i-1} \geq p^m - 1}\big) \vee \vone_{q_i + b_{i-1} \geq p^m}.
        \end{aligned}
    \end{equation}

    To avoid recursive computation, the carry-over can also be derived using:
    \begin{equation}
    \label{eq:iter:add:carry:on}
        \begin{aligned}
            & i_\wedge = \max \{j \leq i \mid q_j + b_{j-1} \geq p^m\}, \\
            & i_\vee = \max \{j \leq i \mid q_j + b_{j-1} \leq p^m - 2\}, \\
            & c_i = \vone_{i_\wedge > i_\vee}.
        \end{aligned}
    \end{equation}

    Alternatively, this can be expressed as:
    \begin{equation}
        c_i = \bigvee_{0 \leq j \leq i} \left[\vone_{q_j + b_{j-1} \geq p^m} \wedge \bigwedge_{j \leq k \leq i} \vone_{q_k + b_{k-1} \geq p^m - 2}\right].
    \end{equation}

    Here, $i_\wedge$ identifies the highest bit contributing a carry to the $i$-th position, while $i_\vee$ identifies the highest bit below $i$ such that carry propagation from bits below $i_\vee$ does not affect higher bits. Thus, $c_i = 1$ if and only if $i_\wedge > i_\vee$.

    Once the carry-over vector $\vc$ is determined, the $p^m$-adic sum can be computed as:
    \begin{equation}
        \begin{aligned}
            & \tilde{o}_0 = q_0, \\
            & \tilde{o}_1 = (q_1 + b_0 + c_0) \mod p^m, \\
            & \cdots \\
            & \tilde{o}_i = (q_i + b_{i-1} + c_{i-1}) \mod p^m.
        \end{aligned}
    \end{equation}

    Finally, to convert the $p^m$-adic representation back to a $p$-adic number $\Tilde{\vo}$, we perform the following modulus operation:
    \begin{equation*}
        o_i = \left\lfloor \frac{\tilde{o}_j \mod p^{(l+1)}}{p^l} \right\rfloor,
    \end{equation*}
    for $i = jk + l$, where $l \in \{0, \dots, k-1\}$ and $j \in \mathbb{Z}$.

    Therefore, the output $\vo$ is precisely the sum of the $k$ input $p$-adic numbers, and the algorithm in \Cref{alg:Iter_ADD} correctly computes $\operatorname{Iter ADD}_p(\va_1, \ldots, \va_k)$ for all $\va_1, \ldots, \va_k$.
\end{proof}

Now we present the proof of \Cref{thm:log-iter-add-con}.

\begin{proof}[Proof of \Cref{thm:log-iter-add-con}]
We demonstrate that a log-precision transformer, with constant depth, a fixed number of attention heads, and constant embedding dimensions, is capable of simulating \Cref{alg:Iter_ADD}. As a result, this model can reliably produce the correct output for any input integers $\va_1, \dots, \va_k$.

    \textbf{Initial Embeddings: } The total length of the input sequence is at most $k(n + 1)$. Tokens in the sequence are divided into two categories: numeric tokens ($0, 1, \dots, p - 1$) and auxiliary tokens ($+$, $=$, \texttt{<SOS>}, \texttt{<EOS>}). Given the parameters $k$ and $n$, we define the parameter $m = \lceil \log_p k \rceil$, as specified in \Cref{alg:Iter_ADD}. The embeddings for these tokens are constructed as follows:

    \begin{itemize}
        \item \textbf{Embedding of input token $a_{i,j}$: } 
        \[
        \ve_{i,j}^0 = \left( a_{i,j}, 0, 0, i, j, j \mod m, \floor{\frac{j}{m}}, p^{j \mod m}, p^{-(j \mod m)}, \ape_{i,j} \right).
        \]

        \item \textbf{Embedding of the $i$-th ``$\operatorname{+}$'' token: } 
        \[
        \ve_{i,+}^0 = \left( 0, 1, 0, i, -1, -1, -1, 0, 0, \ape_{i,+} \right).
        \]

        \item \textbf{Embedding of the ``$\operatorname{=}$'' token: } 
        \[
        \ve_{=}^0 = \left( 0, 1, 0, k+1, -1, -1, -1, 0, 0, \ape_{=} \right).
        \]

        \item \textbf{Embedding of the $\texttt{<SOS>}$ token: } 
        \[
        \ve_\texttt{<SOS>}^0 = \left( 0, 1, 0, 0, -1, -1, -1, 0, 0, \ape_{\texttt{<SOS>}} \right).
        \]

        \item \textbf{Embedding of the $\texttt{<EOS>}$ token: } 
        \[
        \ve_\texttt{<EOS>}^0 = \left( 0, 0, 1, 0, -1, -1, -1, 0, 0, \ape_{\texttt{<EOS>}} \right).
        \]

        \item \textbf{Embedding of output token $o_i$: } 
        \[
        \ve_{o,i}^0 = \left( o_i, 0, 0, 0, i, i \mod m, \floor{\frac{i}{m}}, p^{i \mod m}, p^{-(i \mod m)}, \ape_{o,i} \right).
        \]
    \end{itemize}

Here, $\ape_{\cdots}$ represents the absolute positional encoding. In this construction, the first three dimensions of each embedding correspond to the word embedding, while the remaining six dimensions capture the positional embedding.

\textbf{Block 1.} The first block of the Transformer computes the following quantities:

\begin{enumerate}
    \item $l_{i,j}$: The number of preceding tokens (inclusive) $a_{i',j'}$ satisfying $i' = i$ and $\floor{\frac{j'}{m}} = \floor{\frac{j}{m}}$. The value $l_{i,j}$ is defined only for input number tokens $a_{i,j}$. If undefined, we set $l = -1$.
    \item $f_{i,j}$: Defined as $f_{i,j} = 1$ if no preceding tokens (exclusive) $a_{i',j'}$ exist such that $\floor{\frac{j'}{m}} = \floor{\frac{j}{m}}$; otherwise, $f_{i,j} = 0$. This value is defined only for input number tokens $a_{i,j}$, and if undefined, we set $f = -1$. 
\end{enumerate}

To compute $l_{i,j}$, we define the query, key, value, and $r$ in \Cref{sec:lem-attn} as follows:
\begin{itemize}
    \item Query: $\vq_{i,j} = \left( -1, 2i, -i^2, -1, 2\floor{\frac{j}{m}}, -\floor{\frac{j}{m}}^2 \right)$.
    \item Key: $\vk_{i',j'} = \left( i'^2, i', 1, \floor{\frac{j'}{m}}^2, \floor{\frac{j'}{m}}, 1 \right)$.
    \item Value: $\vv_{i',j'} = (\ape_{i',j'})$.
    \item $r$: $r_{i',j'} = -\ape_{i',j'}$.
\end{itemize}

Using \Cref{lemma:MLP_multip}, the components of the query and key can be computed by preceding MLP layers. The result of the dot product is given by:

\[
\langle \vq_{i,j}, \vk_{i',j'} \rangle = -\left( \floor{\frac{j'}{m}} - \floor{\frac{j}{m}} \right)^2 - (i' - i)^2.
\]

This implies that $\langle \vq_{i,j}, \vk_{i',j'} \rangle = 0$ if $\floor{\frac{j'}{m}} = \floor{\frac{j}{m}}$ and $i = i'$, and $\langle \vq_{i,j}, \vk_{i',j'} \rangle \leq -1$ otherwise. By \Cref{lem:attn-log-copy}, one attention head can be used to copy the absolute position $j''$ of the first token satisfying these conditions. The value of $l_{i,j}$ is then computed as $j - j'' + 1$. 

To compute $f_{i,j}$, we redefine the query, key, and $r$ as follows:
\begin{itemize}
    \item Query: $\vq_{i,j} = \left( -1, 2\floor{\frac{j}{m}}, -\floor{\frac{j}{m}}^2 \right)$.
    \item Key: $\vk_{i',j'} = \left( \floor{\frac{j'}{m}}^2, \floor{\frac{j'}{m}}, 1 \right)$.
    \item Value: $\vv_{i',j'} = (\ape_{i',j'})$.
    \item $r$: $r_{i',j'} = -\ape_{i',j'}$.
\end{itemize}

In this case, the dot product is given by:

\[
\langle \vq_{i,j}, \vk_{i',j'} \rangle = -\left( \floor{\frac{j'}{m}} - \floor{\frac{j}{m}} \right)^2.
\]

This yields $\langle \vq_{i,j}, \vk_{i',j'} \rangle = 0$ if $\floor{\frac{j'}{m}} = \floor{\frac{j}{m}}$, and $\langle \vq_{i,j}, \vk_{i',j'} \rangle \leq -1$ otherwise. By \Cref{lem:attn-log-copy}, one attention head can copy the absolute position $j''$ of the first token satisfying these conditions. The value $f_{i,j}$ is then determined by checking whether $j'' = j$. Specifically, we evaluate:

\[
\vone_{j'' = j} = \text{ReLU}[1 - (j - j'')^2],
\]
which allows $f_{i,j}$ to be computed by a constant-width MLP via \Cref{lemma:MLP_multip}.

Finally, for undefined values, $l$ and $f$ are set to $-1$ in the MLP stage using conditional selection (\Cref{lemma:MLP_select}) based on positional embedding information. In summary, the embeddings generated in this block are $\ve^1 = (l, f)$, and these embeddings are concatenated with the original input embeddings.

\textbf{Block 2.} The second block of the Transformer is designed to compute the first three lines of \Cref{alg:Iter_ADD}. Specifically, this block utilizes the attention mechanism to aggregate the adjacent $m$ bits and derive $s_i$. For each token $a_{i,j}$, let $t_{i,j}$ denote the number of preceding tokens (including $a_{i,j}$ itself) $a_{i',j'}$ such that $\floor{\frac{j}{m}} = \floor{\frac{j'}{m}}$. This block computes the following values:

\begin{enumerate}
    \item $\frac{1}{t_{i,j}}$, where $t_{i,j}$ is as defined above. If $t_{i,j}$ is undefined, its value is set to $-1$.
    \item $c_{i,j}$, the mean value computed over $a_{i',j'} p^{j' \mod m}$ for previous tokens (including $a_{i,j}$) $a_{i',j'}$, where $\floor{\frac{j'}{m}} = \floor{\frac{j}{m}}$. If this value is undefined, it is also set to $-1$.
\end{enumerate}

Using these, we derive $s_w = c_{i, mw} t_{i, mw}$, where $i$ is the largest index such that the length of $\va_i$ exceeds $mk$.

To compute the first value, the query, key, and value vectors are defined as follows:
\begin{itemize}
    \item Query: $\vq_{i,j} = \left( -1, 2\floor{\frac{j}{m}}, -\floor{\frac{j}{m}}^2 \right)$.
    \item Key: $\vk_{i',j'} = \left( \floor{\frac{j'}{m}}^2, \floor{\frac{j'}{m}}, 1 \right)$.
    \item Value: $\vv_{i',j'} = (f_{i',j'})$.
\end{itemize}

Using \Cref{lemma:MLP_multip}, the components of the query and key can be computed by preceding MLP layers. The result of the dot product is given by: $\langle \vq_{i,j}, \vk_{i',j'} \rangle = - \left( \floor{\frac{j'}{m}} - \floor{\frac{j}{m}} \right)^2$. This result implies that $\langle \vq_{i,j}, \vk_{i',j'} \rangle = 0$ when $\floor{\frac{j'}{m}} = \floor{\frac{j}{m}}$, and $\langle \vq_{i,j}, \vk_{i',j'} \rangle \leq -1$ otherwise. By the definition of $f_{i,j}$ and \Cref{lem:attn-log-mean}, the output of the attention mechanism is $\frac{1}{t_{i,j}}$, as required.

To compute the second value, we redefine the query, key, and value vectors as follows:
\begin{itemize}
    \item Query: $\vq_{i,j} = \left( -1, 2\floor{\frac{j}{m}}, -\floor{\frac{j}{m}}^2 \right)$.
    \item Key: $\vk_{i',j'} = \left( \floor{\frac{j'}{m}}^2, \floor{\frac{j'}{m}}, 1 \right)$.
    \item Value: $\vv_{i',j'} = (a_{i',j'} p^{j' \mod m})$.
\end{itemize}

Similar to the computation of the first value, we use \Cref{lemma:MLP_multip} to compute the components of the query and key. In this case, the dot product is given by: $\langle \vq_{i,j}, \vk_{i',j'} \rangle = - \left( \floor{\frac{j'}{m}} - \floor{\frac{j}{m}} \right)^2$. Thus, $\langle \vq_{i,j}, \vk_{i',j'} \rangle = 0$ when $\floor{\frac{j'}{m}} = \floor{\frac{j}{m}}$, and $\langle \vq_{i,j}, \vk_{i',j'} \rangle \leq -1$ otherwise. By applying \Cref{lem:attn-log-mean}, the attention output is $c_{i,j}$, as required.

Finally, for undefined values, we assign $-1$ during the MLP stage by employing conditional selection, as outlined in \Cref{lemma:MLP_select}, utilizing information encoded in the positional embeddings. 

In summary, the new embeddings generated in this block can be expressed as $\ve^2 = \left( \frac{1}{t}, c \right)$. These embeddings are subsequently concatenated with the original embeddings.

\textbf{Block 3.} The third block of the Transformer computes the value of $c_{i,j} t_{i,j}$. This is achieved by first determining $t_{i,j}$ via the attention layer and $\frac{1}{t_{i,j}}$ from the previous block. Subsequently, $c_{i,j} t_{i,j}$ is computed using \Cref{lemma:MLP_multip}.

Notice that $t_{i,j}$ does not exceed the absolute positional value of the current token. We define the query, key, and value vectors as follows:
\begin{itemize}
    \item \textbf{Query:} $\vq_{i,j} = \left( \frac{1}{t_{i,j}^2}, -\frac{2}{t_{i,j}}, 1 \right)$
    \item \textbf{Key:} $\vk_{i',j'} = \left( \ape_{i',j'}^2, \ape_{i',j'}, 1 \right)$
    \item \textbf{Value:} $\vv_{i',j'} = (\ape_{i',j'})$
\end{itemize}
These vectors can be constructed using \Cref{lemma:MLP_multip}. It follows that the inner product 
$$\langle \vq_{i,j}, \vk_{i',j'} \rangle = - \left( \frac{\ape_{i',j'}}{t_{i,j}} - 1 \right)^2,$$
which implies $\langle \vq_{i,j}, \vk_{i',j'} \rangle = 0$ if $\ape_{i',j'} = t_{i,j}$ and $\langle \vq_{i,j}, \vk_{i',j'} \rangle \leq -\frac{1}{n^2k^2}$ otherwise, given that $t_{i,j} \leq nk$. By leveraging \Cref{lem:attn-log-mean}, the attention output is confirmed to be $t_{i,j}$, as required.

Finally, the computation of $c_{i,j} t_{i,j}$ is performed via the subsequent MLP layer. In summary, the embeddings generated in this block are represented as $\ve^3 = (ct)$. These new embeddings are concatenated with the original embeddings to produce the final output of this block.

\textbf{Block 4.} This block of the Transformer corresponds to the fourth step in \Cref{alg:Iter_ADD}, which decomposes $c_{i,j} t_{i,j}$ as $b_{i,j} p^m + q_{i,j}$. It is important to observe that $b_{i,j} \leq i$, meaning that $b_{i,j}$ does not exceed the absolute positional index of the current token. To achieve this decomposition, we define the query, key, and value as follows:
\begin{itemize}
    \item \textbf{Query:} $\vq_{i,j} = \left( -(c_{i,j} t_{i,j} + \frac{1}{2})^2 , 2 p^m (c_{i,j} t_{i,j} + \frac{1}{2}), -p^{2m}\right)$
    \item \textbf{Key:} $\vk_{i',j'} = \left( 1, \ape_{i',j'} - \frac{1}{2}, (\ape_{i',j'} - \frac{1}{2})^2 \right)$
    \item \textbf{Value:} $\vv_{i',j'} = \ape_{i',j'}$
\end{itemize}

The above components can be computed using \Cref{lemma:MLP_multip}. Consequently, the inner product of the query and key is given by:
\begin{equation*}
    \langle \vq_{i,j}, \vk_{i',j'} \rangle = - \left[ c_{i,j} t_{i,j} - \left( \ape_{i',j'} - \frac{1}{2} \right) p^m + \frac{1}{2} \right]^2.
\end{equation*}

This expression implies that $|\langle \vq_{i,j}, \vk_{i',j'} \rangle| \leq \left( \frac{p^m - 1}{2} \right)^2$ if $\ape_{i',j'} = \floor{\frac{c_{i,j} t_{i,j}}{p^m}}$, and $\langle \vq_{i,j}, \vk_{i',j'} \rangle \leq -\left( \frac{p^m + 1}{2} \right)^2$ otherwise. By applying \Cref{lem:attn-log-single-copy}, we obtain:
\begin{equation*}
    c = \frac{(p^m + 1)^2 - (p^m - 1)^2}{(p^m - 1)^2} \geq \frac{4}{p^m}.
\end{equation*}

From this, it follows that $1/c = O(p^m) = O(k)$. Hence, we can design the query, key, and value such that the attention output satisfies $\floor{\frac{c_{i,j} t_{i,j}}{p^m}} = b_{i,j}$. Finally, $q_{i,j}$ can be computed as $q_{i,j} = c_{i,j} t_{i,j} - p^m b_{i,j}$ using the subsequent MLP. The embeddings generated by this block are thus given by $\ve^4 = (b, q)$.

\textbf{Block 5.} This block of the Transformer computes $q_{w+1} + b_w$ for $s_w$. Recall that $s_w = c_{i, mw} t_{i, mw}$, where $i$ is the largest index such that the length of $\va_i$ is greater than $mw$. The goal is to compute these values at their corresponding positions.

First, we use the attention mechanism to copy $q_{w+1}$ for the token $\va_i$ defined above. Note that it is always possible to retrieve the correct value because the position associated with the correct value of $s_{w+1}$ precedes that of $s_k$. To achieve this, we utilize the attention mechanism to copy from the position containing the value $s_{w+1}$. This is implemented by appropriately configuring the query, key, value, and $r$ as described in \Cref{sec:lem-attn}:

\begin{itemize}
    \item Query: $\vq_{i,j} = \left( -1, 2\floor{\frac{j}{m}}, -\floor{\frac{j}{m}}^2, -1 \right)$.
    \item Key: $\vk_{i',j'} = \left( (\floor{\frac{j'}{m}} - 1)^2, \floor{\frac{j'}{m}} - 1, 1, (j' \mod m)^2 \right)$.
    \item Value: $\vv_{i',j'} = (q_{i', j'}, \floor{\frac{j'}{m}})$.
    \item $r$: $r_{i',j'} = \ape_{i', j'}$.
\end{itemize}

The values required for the query or key can be computed using previous MLPs, as shown in \Cref{lemma:MLP_multip}. The dot product $\langle \vq_{i,j}, \vk_{i',j'} \rangle$ evaluates to 
\[
\langle \vq_{i,j}, \vk_{i',j'} \rangle = - \left( \floor{\frac{j'}{m}} - \floor{\frac{j}{m}} - 1 \right)^2 - (j' \mod m)^2.
\]
This implies $\langle \vq_{i,j}, \vk_{i',j'} \rangle = 0$ if $\floor{\frac{j'}{m}} = \floor{\frac{j}{m}} + 1$ and $j' \mod m = 0$, and $\langle \vq_{i,j}, \vk_{i',j'} \rangle \leq -1$ otherwise.

Using \Cref{lem:attn-log-copy}, one attention head suffices to copy the values $q_{i',j'}$ and $\floor{\frac{j'}{m}}$ from the last token satisfying the conditions. Consequently, the first dimension of the attention output equals $q_{w+1}$ if an input number with length greater than $m(w+1)$ exists, as required. Otherwise, $q_{w+1}$ should be zero, and the attention output remains undefined. These two cases can be distinguished by inspecting the second dimension of the attention output: if no input number has a length greater than $m(w+1)$, the second dimension of the attention output is at most $\floor{\frac{j}{m}}$. Using \Cref{lemma:MLP_select}, we can identify these cases and set $q_{w+1} = 0$ when necessary. 

Finally, a subsequent MLP computes the correct value of $q_{w+1} + b_w$ for $s_w$. Additionally, this MLP calculates the indicators $\vone_{q_{w+1} + b_w \geq p^m}$, $\vone_{q_{w+1} + b_w \leq p^m - 2}$, $\vone_{b_w \geq p^m}$, and $\vone_{b_w \leq p^m - 2}$ using the formulation:
\[
\vone_{q_{w+1} + b_w \geq p^m} = \text{ReLU}[q_{w+1} + b_w - (p^m - 1)] - \text{ReLU}[q_{w+1} + b_w - p^m],
\]
as described in \Cref{lemma:MLP_relu}.

To summarize, the embeddings generated in this block are as follows:
\begin{itemize}
    \item For positions with the correct value of $s_w$:
    \[
    \ve^5 = (q_{w+1} + b_w, b_w, \vone_{q_{w+1} + b_w \geq p^m}, \vone_{q_{w+1} + b_w \leq p^m - 2}, \vone_{b_w \geq p^m}, \vone_{b_w \leq p^m - 2}, w).
    \]
    \item For all other positions:
    \[
    \ve^5 = (-1, -1, -1, -1, -1, -1, -1).
    \]
    This can be achieved by filtering out infeasible values using \Cref{lemma:MLP_select}.
\end{itemize}

\textbf{Block 6.} This block of the Transformer computes the following values for positions with the correct value of $s_w$:
\begin{itemize}
    \item The smallest $w_1 \geq w$ such that $\vone_{q_{w+1} + b_w \geq p^m} = 1$. 
    \item The smallest $w_2 \geq w$ such that $\vone_{q_{w+1} + b_w \leq p^m - 2} = 1$.
\end{itemize}
Both calculations rely on the standard COPY operation, which can be implemented using \Cref{lem:attn-log-copy}. To ensure the validity of $w_1$ and $w_2$ (i.e., the existence of such indices), we COPY the values $\vone_{q_{w_1+1} + b_{w_1} \geq p^m}$ and $\vone_{q_{w_2+1} + b_{w_2} \leq p^m - 2}$, verifying that they equal $1$. Invalid values can then be filtered out using an MLP, as described in \Cref{lemma:MLP_select}. 

The embeddings generated in this block are as follows:
\begin{itemize}
    \item For positions with the correct value of $s_w$: $\ve^6 = (w_1, w_2)$.
    \item For all other positions: $\ve^6 = (-1, -1)$. (This can be achieved by filtering infeasible values using \Cref{lemma:MLP_select}.)
\end{itemize}

\textbf{Block 7.} The last block of the Transformer executes the final four steps of \cref{alg:Iter_ADD}. This layer calculates the carry-over bits $\vc$ and $p^m$-adic representation of the final output $\vo$ via the attention mechanism and the MLP, subsequently converting the $p^m$-adic number into a $p$-adic number.

The computation of carry-on bits, as described in \Cref{eq:iter:add:carry:on} within \cref{alg:Iter_ADD}, adheres to the following equations:
\begin{equation}
       \begin{aligned}
            & i_\wedge=\max \{w\leq i \mid q_w+b_{w-1}\geq p^m\}, \\
            & i_\vee = \max \{w\leq i \mid q_w+b_{w-1}\leq p^m-2\}, \\
            & c_i = \vone_{i_\wedge > i_\vee}.
        \end{aligned}
    \end{equation}
In the attention layer, operations are restricted to output tokens and other tokens will maintain the embeddings via the residual connection and the filter operation by MLP. Let's consider the token $o_{(i+1)m+j+1}$, where $j \in \{0,\cdots, m - 1\}$, we want to predict the next token ${o}_{(i+1)k+j}$. The model executes the COPY operation, duplicating the previous embeddings to extract $q_{i+1} + b_i$, $i_{\wedge}$, and $i_{\vee}$. The extraction is similar to previous blocks, but here we only need to focus on positions with correct value of $s_w$. To find out the value of $i_{\wedge}, i_{\vee}$, we first COPY the embedding of the position with the correct value of $s_i$, and find the minimum $w'$ which shares the same value of $w_1,w_2$ with $s_i$. Again, this can be implemented by several COPY operation with \Cref{lem:attn-log-copy}.

The carry-over bit $c_i$ and the $p^m$-adic results $\tilde o_{i+1}$ are then computed as follows:
\begin{equation*}
c_i = \vone_{i_\wedge > i_\vee}, \;\tilde o_{i+1} = b_i + c_i + q_{i+1}.
\end{equation*}

This computation is facilitated by a constant-size MLP. Subsequently, for the output token $\Tilde{o}_{(i+1)k+j}$, the result $o_{(i+1)k+j} = \tilde o_{i+1} \mod p^{j+1}$ is required. We first calculate $\tilde o_{i+1} / p^{j+1}$ using the positional embedding and \Cref{lemma:MLP_multip}, then calculate $\floor{\tilde o_{i+1} / p^{j+1}}$ using the similar fashion to what we did in Block 4, and then calculate $\tilde o_{i+1} \mod p^{j+1}$ using MLP. Finally, we can get the value of $\floor{\frac{\tilde o_{i+1} \mod p^{j+1}}{p^j}}$ using the similar fashion to what we did in Block 4.

Upon outputting the token ${o}_0$, the model anticipates the \texttt{<EOS>} token, employing an MLP to filter the hidden embeddings and output the word embedding for \texttt{<EOS>}. 
Thus, the final output from this layer is characterized by the equation:
\begin{equation*}
    \ve_{o,i}^7 = \begin{cases}
    ({o}_{i-1},i,0) & \text{if \(i > 0\)}, \\
    (-1, -1, 1) & \text{if \(i = 0\)}.
\end{cases}
\end{equation*}

    \textbf{Predict Next Token. } Given the output embeddings of the last transformer layer $\ve_{o,i}^7$, and the word embeddings, the transformer can simply predict the next token by finding the nearest word embeddings.

    In this construction, the norm of the parameters is bounded by $poly(n,k)$, therefore, this construction can be implemented by a log-precision transformer with arbitrarily small error.
\end{proof}

\subsection{Proof for Theorem \ref{thm:log-mul-mod-con}}
\label{sec:proof-log-mul-mod-con}
\begin{thmbis}{thm:log-mul-mod-con}
Fix integers $p\geq 2$ and $c\in \mathbb{N}^*$. Consider the tokenizer $\mT_c$ defined in \cref{eq:tokenizer} for processing the input and output sequences. For any integers $n$ and $l\leq 2n$, there exists a logarithmic-precision Transformer with constant depth (independent of $n$ and $k$) and hidden dimensions $O(n^2)$ that can generate the correct output for any input on the $\operatorname{MUL}_p(n, l)$ task.

\end{thmbis}

Here, we first describe an algorithm to perform $\operatorname{MUL}_p(n, l)$ (\Cref{alg:Mul}) and prove the correctness of \Cref{alg:Mul}. Then, we construct a Transformer with the configurations in \Cref{thm:log-mul-mod-con} capable for simulating \Cref{alg:Mul}.

\begin{algorithm}
 \SetKwInOut{Input}{Input}
 \SetKwInOut{Output}{Output}
\caption{$\operatorname{MUL}_p(n, l)$ Algorithm}\label{alg:Mul}
\SetAlFnt{\small\sf}
   \Input{ Two $p$-adic numbers $\va, \vb$ no longer than $n$ bits, truncating length $l$}
   \Output{ ${\vo} := \va \vb \mod p^l$}
   \BlankLine
   $m = \lceil \log_p n \rceil + 1$\;
   Compute the product of each pair of bits: $d_{i,j} = a_i b_j$\;
   Compute each bit as $$r_j = \sum_{k=\max(0, j - (n-1))}^{\min(n-1, j)} d_{k,j-k}$$
   for $j = 0, \cdots, 2n-1$\;
   Combine neighboring $m$ bits: $$s_i=\sum_{j=0}^{m-1} r_{ik+j}p^{j}$$
   for $i = 0, \cdots, \lfloor (2n-1) / m \rfloor$\;
   Decompose $s_i$ by $s_i = b_i p^m + q_i$, where $q_i\in[0,p^m-1]$ and $b_i,q_i\in\mathbb{N}$\;
   $b_{-1} = 0$\;
   \ForEach{$i = 0, \cdots, \lfloor (2n-1) / m \rfloor$}{
   $f_i = \vone_{q_i + b_{i-1} \geq p^m}$\;
   $g_i = \vone_{q_i + b_{i-1} \geq p^m -2}$\;
   }    
   Compute the carry-on bits $\vc$:
   $$c_i = \bigvee_{0 \leq j \leq i} \left(f_j \wedge \bigwedge_{j \leq k \leq i} g_k\right)$$
   for $i = 0, \cdots, \lfloor (2n-1) / m \rfloor$\;
   
   Compute the $p^m$-adic outcome $\tilde \vo$: $\tilde{o}_i=(q_i + b_{i-1} + c_{i-1})\mod p^m$ for $i = 0, \cdots, \lfloor (2n-1) / m \rfloor$\;
   Covert $p^m$-adic $\tilde \vo$ to $p$-adic $\vo$: 
   $${o}_i = \left\lfloor \frac{\tilde o_j\mod p^{(l+1)}}{p^l} \right\rfloor$$ for $i=jk+l$ where $l\in \{0, \cdots, k-1\}$, $j\in \sZ$\;
\end{algorithm}

\begin{lemma}[An algorithm to perform $\operatorname{MUL}_p(n, l)$] 
\Cref{alg:Mul} outputs $\vo = \va\vb \mod p^l$ for all inputs $\va,\vb$.
\label{lmm:alg:mul}
\end{lemma}
\begin{proof}
    It's easy to verify $\sum_{i} s_i p^{im}$ accurately represents the product of $\va, \vb$. For the subsequent steps, the proof is the same as that of \Cref{lmm:alg:iter:add} since they share the same procedures.
\end{proof}

Next, we provide the proof for \Cref{thm:log-mul-mod-con}.

\begin{proof}[Proof for \Cref{thm:log-mul-mod-con}]
    Now, we demonstrate that a log-precision transformer, with a constant depth, a fixed number of attention heads, and $O(n^2)$ embedding dimensions, is capable of simulating \Cref{alg:Mul}. Consequently, this model can accurately generate correct output for any input integers $\va, \vb$.

    \textbf{Initial Embeddings: } The total length of the input sequence is no longer than $2(n + 1)$. We categorize the tokens into two classes: number tokens ($0,1,\cdots, p - 1$) and auxiliary tokens ($+$, $=$, \texttt{<SOS>} and \texttt{<EOS>}). Given the parameters $k,n$, we determine the parameter $m = \lceil \log_p k \rceil + 1 \geq 2$, as specified in \cref{alg:Mul}. The embeddings for these classes are defined as follows:

    \begin{itemize}
        \item \textbf{Embedding of input token $a_{i}$: } $\vu_{a,i}^0 = \left( a_i \ve_{i+1}, 0, -1, -1,  0, 1, i, 0, \ape_{a, i} \right)$.
        \item \textbf{Embedding of input token $b_{i}$: } $\vu_{b,i}^0 = \left( 0, b_i \ve_{i+1}, -1,  -1,0, 2, i, 0, \ape_{b, i} \right)$.
        \item \textbf{Embedding of the ``$\times$'' token: } $\vu_{\times}^0 = (-1, -1, -1,  -1,-1, 4, -1, 0, \ape_{\times})$.
        \item \textbf{Embedding of the ``$\operatorname{=}$'' token: } $\vu_{=}^0 = (-1, -1, -1,  -1,-1, 5, -1, 0, \ape_{=})$.
        \item \textbf{Embedding of the $\texttt{<SOS>}$ token: } $\vu_\texttt{<SOS>}^0 = (-1, -1, -1, -1, -1, 6, -1, 0, \ape_{\texttt{<SOS>}})$.
        \item \textbf{Embedding of the $\texttt{<EOS>}$ token: } $\vu_\texttt{<EOS>}^0 = (-1, -1, -1, -1, -1, 7, -1, 0, \ape_{\texttt{<EOS>}})$.
        \item \textbf{Embedding of output token $o_i$: }  $\vu_{o,i}^0 = ( -1, -1, o_i, \ve_{\floor{i/m}}, -1, 3, i, p^{-(i \mod m)}, \ape_{o,i})$.
    \end{itemize}
    where $\ve_i \in \sR^{n}$ is one-hot vector, and $\ape_{\cdots}$ is absolute positional embedding.
    In this construction, the first $3n + 3$ dimensions of each initial embedding represent the word embedding, while the last three dimensions accounts for the position embedding.

    \textbf{Block 1.} The first block of the Transformer executes the first three lines of \Cref{alg:Mul}. To be specific, we first aggregate the input number $\va, \vb$ to the positions of $b_0$, and then calculate the values of $r_j$.

    To aggregate the input number $\va, \vb$ to the positions of $b_0$, we set the query, key and value as follows:
    \begin{itemize}
        \item Query: $\vq = (\ve^0[2n+2])$, i.e., $\vq = (0)$ for input number $\va, \vb$, and $\vq = (-1)$ otherwise.
        \item Key: $\vk = (1)$.
        \item Value: $\vv = \ve^0[1, \cdots, 2n]$. 
    \end{itemize}
    Thus $\langle \vq, \vk\rangle = 0$ for key of input number tokens, and $\langle \vq, \vk \rangle \leq  -1$ otherwise. By \Cref{lem:attn-log-mean}, the attention output is 
    \begin{equation*}
        \frac{1}{\ape_{b,0} - 2} (a_0, \cdots, a_{n-1}, b_0, \cdots, b_{n-1}).
    \end{equation*}
    By \Cref{lemma:MLP_multip}, we can use the subsequent MLP to get $(a_0, \cdots, a_{n-1}, b_0, \cdots, b_{n-1})$ given the value of $\ape_{b,0}$. Then we can calculate all $d_{i,j}$ using the MLP, which requires $O(n^2)$ hidden dimension by \Cref{lemma:MLP_multip}.
    
    Finally, we calculate $(r_{2n-1}, \cdots, r_0)$ by 
    \begin{equation*}
    r_j = \sum_{k=\max(0, j - (n-1))}^{\min(n-1, j)} d_{k,j-k}.
    \end{equation*}

    \textbf{Block 2.} This block of the Transformer uses several MLPs to executes line 4-12 of \Cref{alg:Mul}. All the calculations below are also calculated at the position of $b_0$, subsequent to what we did in Block 1.
    \begin{itemize}
        \item For the calculation of $s_i$, it's easy to get the values via $(r_{2n-1}, \cdots, r_0)$.
        \item For the calculation of $b_i,q_i$, notice that $b_i \leq p^m \leq np^2$, thus we can use 
        \begin{equation*}
            b_i = \sum_{j = 0}^{np^2} \text{ReLU}(s_i - p^m)
        \end{equation*}
        for each $b_i$, which requires $O(n^2)$ hidden dimension in total by \Cref{lemma:MLP_relu}. Then $q_i = s_i - b_i p^m$, which can be easily implemented by MLP as well.
        \item For the calculation of $f_i, g_i$, we can get those values by 
        \begin{equation*}
        \begin{aligned}
            & f_i = \text{ReLU}[q_i + b_{i-1} - (p^m - 1)] - \text{ReLU}[q_i + b_{i-1} - p^m], \\
            & g_i = \text{ReLU}[q_i + b_{i-1} - (p^m - 2)] - \text{ReLU}[q_i + b_{i-1} - (p^m - 1)] 
        \end{aligned}
        \end{equation*}
        and \Cref{lemma:MLP_relu}, which requires $O(n)$ hidden dimension in total. 
        \item For the calculation of $c_i$, notice that 
        \begin{equation*}
            \bigwedge_{1\leq i\leq \gamma} \alpha_i = \text{ReLU}\left( \sum_{i=1}^\gamma \alpha_i - \gamma + 1 \right), \;
            \bigvee_{1 \leq i\leq \gamma} \alpha_i = 1 - \text{ReLU} \left( 1 - \sum_{i=1}^\gamma \alpha_i \right).
        \end{equation*}
        Combining with \Cref{lemma:MLP_relu}, we can calculate the value of each $c_i$ with $O(n)$ hidden dimension. 
        \item Finally, for the calculation of $\tilde o_i$, we can use the similar fashion of the calculation of $q_i$. Since $q_i + b_{i-1} + c_{i-1} < 2p^m$, we can calculate each $\tilde o_i$ using constant hidden dimension, which implies we can calculate $\tilde \vo$ using $O(n)$ hidden dimension in total.
    \end{itemize}

    \textbf{Block 3.} The last block of the Transformer executes the last step of \cref{alg:Mul}. Let's consider the token $o_{(i+1)m+j+1}$, where $j \in \{0,\cdots, m - 1\}$, we want to predict the next token $o_{(i+1)k+j}$. We first COPY the value of $\tilde \vo$ from the position of $b_0$, then extracts $\tilde o_{i+1}$ by $\tilde o_{i + 1} = \langle \tilde \vo, \ve_{i + 1} \rangle$
    using the positional embedding of $\vu_{o,i}^0$. 
    
    Subsequently, for the output token ${o}_{(i+1)k+j}$, the result $o_{(i+1)k+j} = \tilde o_{i+1} \mod p^{j+1}$ is required. We first calculate $o_{i+1} / p^{j+1}$ using the positional embedding and \Cref{lemma:MLP_multip}, then calculate $\floor{\tilde o_{i+1} / p^{j+1}}$ using the similar fashion to what we did when calculating $s_i, b_i$ in Block 2. Since $\tilde o_{i+1} < 2p^m \leq np^2$, this can be implemented by a MLP with $O(n)$ hidden dimension. 
    Then we can calculate $\tilde o_{i+1} \mod p^{j+1}$ using MLP. Similarly, we can finally get the value of $\floor{\frac{\tilde o_{i+1} \mod p^{j+1}}{p^j}}$ using a MLP with $O(n)$ hidden dimension.

Upon outputting the token ${o}_0$, the model anticipates the \texttt{<EOS>} token, employing an MLP to filter the hidden embeddings and output the word embedding for \texttt{<EOS>}. 
Thus, the final output from this layer is characterized by the equation:
\begin{equation*}
    \ve_{o,i}^3 = \begin{cases}
    ({o}_{i-1},i,3) & \text{if \(i > 0\)}, \\
    (-1, -1, 7) & \text{if \(i = 0\)}.
\end{cases}
\end{equation*}

    \textbf{Predict Next Token. } Given the output embeddings of the last transformer layer $\ve_{o,i}^3$, and the word embeddings, the transformer can simply predict the next token by softmax.

    In this construction, the norm of the parameters is bounded by $O(n^2)$, therefore, this construction can be implemented by a log-precision transformer with arbitrarily small error.
\end{proof}

\section{Experimental Details}
\label{app:experimental-details}

In this section, we present the experimental details.

\subsection{Datasets}
\label{sec:gen_data}
The iterated addition and integer addition data are generated according to \Cref{alg:itadd-implement}. The multiplication data are generated according to \Cref{alg:mult-implement}. Both datasets are used online for training and testing.
\newpage

\begin{algorithm}[H]
    \caption{Iterated Addition Data Generation}
    \label{alg:itadd-implement}
    \SetKwFunction{largeNumberAdd}{large\_number\_add}
    \SetKwFunction{getData}{get\_data}

    \SetKwProg{Fn}{Function}{:}{}
    \Fn{\largeNumberAdd{$a, b, base$}}{
        \textbf{Input:} $a$: List of digits of the first number\\
        \phantom{\textbf{Input:}} $b$: List of digits of the second number\\
        \phantom{\textbf{Input:}} $base$: The numerical base\\
        \textbf{Output:} $result$: List of digits of the sum of $a$ and $b$

        carry $\gets 0$, result $\gets [ ]$\\
        max\_length $\gets$ max(length(a), length(b))\\
        \For{$i\gets 0$ \KwTo \text{max\_length - 1}}{
            sum $\gets$ carry\\
            \If{$i<$ length(a)}{
                sum $\gets$ sum + a[i]\\
            }
            \If{$i<$ length(b)}{
                sum $\gets$ sum + b[i]\\
            }
            carry $\gets$ floor(sum / base)\\
            result.\text{append}(sum $\mod$ \text{base}) \\
        }
        \If{carry $\neq 0$}{
            result.\text{append}(carry)\\
        }
        \KwRet result\\
    }
    
    \Fn{\getData{$batch, length, num\_count, base$}}{
        \textbf{Input:} \\
        $batch$: Number of samples \\
        $length$: Maximum length of addends \\
        $num\_count$: Number of addends \\
        $base$: The numerical base \\
        \textbf{Output:} tokenized\_data: Tensor of generated sequences\\
        data $\gets$ \text{random integers in range }[0, base) \text{ with shape } (batch, length, num\_count)\\
        tokenized\_data $\gets []$\\

        \For{$i\gets 0$ \KwTo $batch - 1$}{
            numbers $\gets$ data[i, :, :]\\
            strip leading zeros of numbers and get stripped\_numbers\\
            \For{num in numbers}{sum\_digits $\gets$ \text{large\_number\_add}(sum\_digits, num, base)\\}
            reverse stripped\_numbers and sum\_digits\\
            add token of '+' and '=' and '<EOS>' to form sequence
            pad the sequence into the same length\\
            tokenized\_data.append(sequence)\\
        }
        convert tokenized\_data to tensor\\
        \KwRet tokenized\_data\\
    }
\end{algorithm}

\begin{algorithm}[H]
    \caption{Integer Multiplication Data Generation}
    \label{alg:mult-implement}
    \SetKwFunction{largeNumberMult}{large\_number\_mult}
    \SetKwFunction{getMultData}{get\_mult\_data}

    \SetKwProg{Fn}{Function}{:}{}
    \Fn{\largeNumberMult{$a, b, base$}}{
        \textbf{Input:} $a$: List of digits of the first number\\
        \phantom{\textbf{Input:}} $b$: List of digits of the second number\\
        \phantom{\textbf{Input:}} $base$: The numerical base\\
        \textbf{Output:} $result$: List of digits of the product of $a$ and $b$\\
        
        result $\gets$ [0] * (length(a) + length(b))\\
        \For{$i \gets 0$ \KwTo $length(a) - 1$}{
            carry $\gets$ 0\\
            \For{$j \gets 0$ \KwTo $length(b) - 1$}{
                product $\gets$ $a[i] * b[j] + \text{result}[i+j] $ + carry\\
                carry $\gets$ \text{floor}(product / base)\\
                result$[i+j]$ $\gets$ product $\mod$ base\\
            }
            \If{carry > 0}{
                result[i + length(b)] $\gets$ result[i + length(b)] + carry\\
            }
        }
        strip leading zeros from result\\
        \KwRet result\\
    }
    \Fn{\getMultData{$batch, length, base$}}{
        \textbf{Input:} \\
        $batch$: Number of samples \\
        $length$: Maximum length of multiplicands \\
        $base$: The numerical base \\
        \textbf{Output:} tokenized\_data: Tensor of generated sequences\\
        data $\gets$ \text{random integers in range }[0, base) \text{ with shape } (batch, length, 2)\\
        tokenized\_data $\gets []$\\

        \For{$i\gets 0$ \KwTo $batch - 1$}{
            num\_1 $\gets$ data[i, :, 0]\\
            num\_2 $\gets$ data[i, :, 1]\\
            strip leading zeros of numbers and get stripped\_numbers\\
            product\_digits $\gets$ \text{large\_number\_mult}(num\_1, num\_2, base)\\
            reverse stripped\_numbers and product\_digits\\
            add token of '$\times$' and '=' and '<EOS>' to form sequence
            pad the sequence into the same length\\
            tokenized\_data.append(sequence)\\
        }
        convert tokenized\_data to tensor\\
        \KwRet tokenized\_data\\
    }
\end{algorithm}

\subsection{Model Training}
\label{sec:training}
    The experiments were conducted on a single NVIDIA GeForce RTX 4090 GPU over a duration of two weeks, investigating the differences in performance between standard precision and low precision operations. To avoid some unexpected issues of hardware, we also conduct the same experiments on NVIDIA A100 GPUs, and the results are consistent with the results on NVIDIA GeForce RTX 4090 GPU. We try 3 different seeds and select the maximum accuracy for each task. 

    The model configuration in our experiments is presented in \cref{tab:model_config}, and the training configuration is presented in \cref{tab:training_config}.

\begin{table}[H]
    \centering
    \begin{tabular}{ll}
      \toprule
    \multicolumn{2}{l}{\textbf{Model Configuration} } \\
    \midrule
        Model Depth & $\{3,5\}$ \\
        Hidden Dimension & 256 \\
        Attention Heads & 4\\
        Positional Embeddings & RoPE\\
        Activation & NewGeLU\\
        \bottomrule
    \end{tabular}
    \caption{Model Configuration for Transformer in Experiments.}
    \label{tab:model_config}
\end{table}

\begin{table}[H]
    \centering
    \begin{tabular}{ll}
      \toprule
    {\centering \textbf{Training Configuration} }\\
    \midrule
        Epochs & 1\\
        Learning Rate & 1e-3 \\
        Optimizer & AdamW \\
        $\beta_1$ & 0.9 \\
        $\beta_2$ & 0.999 \\
        Weight Decay & 0.01\\
        Learning Rate Scheduler & Cosine Scheduler with Warmup \\
        Numerical Precision & $\{\texttt{float32},\texttt{bfloat16}\}$\\
        \bottomrule
    \end{tabular}
    \caption{Training Configuration in Experiments.}
    \label{tab:training_config}
\end{table}

\subsection{Integer Addition Results}
\label{sec:int_add_result}
The results of the experiments are presented in \cref{tab:intadd-results}.

\begin{table}[H]
    \centering
    \begin{tabular}{ccccc}
      \toprule
      & \multicolumn{2}{c}{\textbf{Base-2}} & \multicolumn{2}{c}{\textbf{Base-10}}\\
    \textbf{Length} & \texttt{float32} \textbf{Accuracy} & \texttt{bfloat16} \textbf{Accuracy} & \texttt{float32} \textbf{Accuracy} & \texttt{bfloat16} \textbf{Accuracy}\\
    \midrule
    8 & 99.8\% & 99.6\% & 99.4\% & 99.0\% \\
    16 & 99.3\% & 98.4\% & 99.2\% & 98.1\% \\
    24 & 98.9\% & 96.3\% & 99.2\% & 97.4\% \\
    32 & 99.3\% & 95.9\% & 99.2\% & 94.1\% \\
        \bottomrule
    \end{tabular}
    \caption{Evaluation of integer addition accuracy across various length with both 32-bit and 16-bit precision.}
    \label{tab:intadd-results}
\end{table}

\subsection{Fine-tuing Configuration, Generation Configuration, and Prompt For LLM}
\label{sec:detail-llm}
    The fine-tuning configuration and generation configuration for LLMs is listed in \cref{tab:lora_config,tab:generation_config}. The detailed prompts for the three elementary arithmetic tasks are listed in the \cref{tab:prompt_add,tab:prompt_mul} and generation configuration can be found in the \cref{tab:generation_config}.

\begin{table}[H]
    \centering
    \begin{tabular}{ll}
      \toprule
    {\centering \textbf{Generation Configuration} }\\
    \midrule
        TopK & 50 \\
        TopP & 0.95 \\
        Temperature & 0.1\\
        \bottomrule
    \end{tabular}
    \caption{Generation Configuration for LLAMA 3.1 8B Instruct in arithmetic tasks.}
    \label{tab:generation_config}
\end{table}

\begin{table}[H]
    \centering
    \begin{tabular}{ll}
      \toprule
    {\centering \textbf{Fine-tuning Configuration} }\\
    \midrule
        Rank & 8 \\
        Scaling Factor & 16 \\
        Dropout Rate & 0.05\\
        Epochs & 1\\
        Learning Rate & 2e-4 \\
        Optimizer & AdamW \\
        $\beta_1$ & 0.9 \\
        $\beta_2$ & 0.999 \\
        Weight Decay & 0.01\\
        Learning Rate Scheduler & Cosine Scheduler with Warmup \\
        Warmup Ratio & 0.1 \\
        Numerical Precision & $\{\texttt{bfloat16},\texttt{int4}\}$\\
        \bottomrule
    \end{tabular}
    \caption{Generation Configuration for LLAMA 3.1 8B Instruct in arithmetic tasks.}
    \label{tab:lora_config}
\end{table}

\begin{table}[H]
    \centering
    \begin{tabular}{p{0.97\columnwidth}}
    \toprule
    {\centering \textbf{Prompt for LLAMA 3.1 8B Instruct in Integer Addition and Iterated Addition tasks.} }\\
    \midrule
       Please directly calculate the following arithmetic expression in base <base> with the following format:\\
<Expression> = <Result>\\
It is important that you should not show any intermediate steps in your calculation process.\\
The final answer should be computed in one step and provided the final result immediately without any explanation.\\
Here are some examples\\
32 + 78= 110\\
1234 + 4567 + 2134 + 4567 = 12502\\
2135 + 523 + 2135 + 523 = 5316\\
2314 + 4567 + 2314 + 4567 = 13762\\
Arithmetic Expression:\\
<Expression>\\
\bottomrule
    \end{tabular}
    \caption{Prompt for LLAMA 3.1 8B Instruct in Integer Addition and Iterated Addition tasks.}
    \label{tab:prompt_add}
\end{table}

\begin{table}[H]
    \centering
    \begin{tabular}{p{0.97\columnwidth}}
    \toprule
    {\centering \textbf{Prompt for LLAMA 3.1 8B Instruct in Integer Multiplication task.} }\\
    \midrule
       Please directly calculate the following arithmetic expression in base <base>.\\
It is important that you should not show any intermediate steps in your calculation process.\\
The final answer should be computed in one step and provided the final result immediately without any explanation.\\
Here are some examples\\
Examples:\\
32 * 56 = 1792\\
867 * 467 = 404889\\
123 * 456 = 56088\\
Arithmetic Expression:\\
<Expression>\\
\bottomrule
    \end{tabular}
    \caption{Prompt for LLAMA 3.1 8B Instruct in Integer Multiplication task.}
    \label{tab:prompt_mul}
\end{table}

\newpage

\subsection{Reference Results for LLMs}
We also provide the results of GPT-4o and GPT-4o-mini as a baseline for these arithmetic tasks base-10 for reference. The results are presented in \cref{tab:gpt4}.

\begin{table}[H]
    \centering
    \begin{tabular}{cccc}
      \toprule
    \textbf{Task} & \textbf{Length} & \textbf{GPT-4o} & \textbf{GPT-4o-mini}\\
    \midrule
    \multirow{5}{*}{Addition of 2 numbers} & 1 & 100.0\% & 100.0\% \\
    & 4 & 99.9\% & 98.8\% \\
    & 7 & 97.5\% & 51.4\%\\
    & 10 & 96.3\% & 46.0\%\\
    & 13 & 93.3\% & 44.0\%\\
    \midrule
    \multirow{5}{*}{Addition of 3 numbers} & 1 & 100.0\% & 100.0\% \\
    & 3 & 99.8\% & 99.6\% \\
    & 5 & 98.9\% & 73.4\%\\
    & 7 & 69.2\% & 9.1\%\\
    & 9 & 38.5\% & 5.8\%\\
     \midrule
    \multirow{5}{*}{Addition of 5 numbers} & 1 & 100.0\%  & 100.0\%  \\
    & 2 & 100.0\% & 99.4\%  \\
    & 3 & 100.0\%  & 89.5\% \\
    & 4 & 88.4\% & 31.1\%  \\
    & 5 & 86.8\% & 24.7\%  \\
     \midrule
    \multirow{5}{*}{Multiplication of 2 numbers} & 1 & 100.0\% & 100.0\%\\
    & 2 & 100.0\% & 97.5\% \\
    & 3 & 76.6\% & 44.7\%\\
    & 4 & 21.5\% & 7.6\% \\
    & 5 & 4.1\% & 0.7\%\\
    \bottomrule
    \end{tabular}
    \caption{The Performance of GPT-4o and GPT-4o-mini on the arithmetic tasks.}
    \label{tab:gpt4}
\end{table}

\end{document}